%% file: wsdm25.tex
\documentclass[sigconf,table]{acmart}

\usepackage{amsmath}  
\usepackage{amsfonts} 

\usepackage{amssymb}
\usepackage{amsthm}   

\theoremstyle{definition}
\newtheorem{definition}{Definition}
\usepackage{booktabs}
\usepackage{enumitem}

\usepackage[linesnumbered,ruled,vlined]{algorithm2e}
\usepackage{multirow}

\usepackage{mdframed}

\usepackage{pgfplots}
\usepackage{tikz}
\pgfplotsset{compat=1.17}

\usepackage{graphicx}
\usepackage{subcaption}
\usepackage{thmtools}

\usepackage[normalem]{ulem}

\definecolor{lightgreen}{RGB}{242, 244, 232}
\definecolor{darkgreen}{RGB}{113, 156, 62}

\newmdenv[
  linecolor=darkgreen,
  backgroundcolor=lightgreen,
  leftline=true,
  rightline=false,
  topline=false,
  bottomline=false,
  linewidth=4pt,
  innerleftmargin=15pt,
  innerrightmargin=10pt,
  innertopmargin=10pt,
  innerbottommargin=10pt,
  skipabove=\baselineskip,
  skipbelow=\baselineskip
]{remarkbox}

\newenvironment{remark}{%
  \begin{remarkbox}
  \noindent\textbf{Remarks:}%
}{%
  \end{remarkbox}%
}

\usepackage{url}
\usepackage{hyperref} 
 \hypersetup{ 
     colorlinks=true, 
     linkcolor=black, 
     filecolor=black, 
     citecolor = black,       
     urlcolor=blue, 
     pdfborderstyle={/S/U/W 1}
     }

\setcopyright{acmlicensed}
\copyrightyear{2025} 
\acmYear{2025} 
\acmDOI{XXXXXXX.XXXXXXX}
\acmISBN{978-1-4503-XXXX-X/2018/06}

\begin{document}

\title{Graph Neural Diffusion via Generalized Opinion Dynamics}

\author{Asela Hevapathige}
\email{asela.hevapathige@anu.edu.au}
\affiliation{%
  \institution{School of Computing \\ Australian National University}
  \city{Canberra}
  \country{Australia}
}

\author{Asiri Wijesinghe}
\email{asiriwijesinghe.wijesinghe@data61.csiro.au}
\affiliation{%
  \institution{Data61 \\ CSIRO}
  \city{Canberra}
  \country{Australia}
}

\author{Ahad N. Zehmakan}
\email{ahadn.zehmakan@anu.edu.au}
\affiliation{%
  \institution{School of Computing \\ Australian National University}
  \city{Canberra}
  \country{Australia}
}


\begin{abstract}
There has been a growing interest in developing diffusion-based Graph Neural Networks (GNNs), building on the connections between message passing mechanisms in GNNs and physical diffusion processes. However, existing methods suffer from three critical limitations: (1) they rely on homogeneous diffusion with static dynamics, limiting adaptability to diverse graph structures; (2) their depth is constrained by computational overhead and diminishing interpretability; and (3) theoretical understanding of their convergence behavior remains limited. To address these challenges, we propose \emph{GODNF}, a Generalized Opinion Dynamics Neural Framework, which unifies multiple opinion dynamics models into a principled, trainable diffusion mechanism.  Our framework captures heterogeneous diffusion patterns and temporal dynamics via node-specific behavior modeling and dynamic neighborhood influence, while ensuring efficient and interpretable message propagation even at deep layers. We provide a rigorous theoretical analysis demonstrating GODNF’s ability to model diverse convergence configurations. Extensive empirical evaluations of node classification and influence estimation tasks confirm GODNF's superiority over state-of-the-art GNNs. 
\end{abstract} 



\maketitle

\input{Sections/introduction}
\input{Sections/related_work}
\input{Sections/background}
\input{Sections/methodology}

\input{Sections/theoretical_analysis}

\input{Sections/experiments}\input{Sections/conclusion}



\bibliography{wsdm25}
\bibliographystyle{ACM-Reference-Format}
\clearpage

\input{Sections/appendix}

\end{document}

%% file: Sections/introduction.tex
\section{Introduction}

Graph Neural Networks (GNNs) have become a principal technique for graph learning tasks due to their ability to integrate node features and structural information into meaningful representations. They have been successfully applied to a broad
range of problems in various domains, including social network analysis \cite{li2023survey,jain2023opinion,lyu2023dcgnn}, recommendation systems \cite{wu2022graph,gao2022graph, cai2023expressive}, molecular property prediction \cite{bodnar2021weisfeiler,wieder2020compact,wu2023chemistry}, and fraud detection \cite{kim2022graph,li2023lgm,hyun2024lex}. Many GNNs are built on a message passing mechanism, where nodes iteratively exchange information with neighbors to refine their representations, allowing them to model the relational structure of graphs effectively \cite{wu2020comprehensive}.

Recently, researchers have increasingly focused on formulating the message passing process in GNNs as a diffusion phenomenon to better understand and improve the propagation of information in graph data \cite{gasteiger2019diffusion,zhao2021adaptive,chamberlain2021grand}. These models are inspired by physical diffusion processes \cite{hancontinuous}, offering distinct advantages over traditional GNNs as they explicitly model the dynamics of information flow across networks. Although traditional GNNs rely mainly on aggregation operations in the local neighborhood \cite{kipf2017semi,velivckovic2018graph}, diffusion-based GNNs leverage principles from random walks, heat diffusion, and other physical processes, enabling them to capture local and global structural dependencies in graphs in a more effective manner \cite{lee2023time,gasteiger2019diffusion}. 


Despite their strengths, diffusion GNNs face several notable limitations:
\begin{itemize}
    \item \textbf{Homogeneous Diffusion with Static Dynamics:} Traditional diffusion GNNs assume uniform information propagation across all nodes and fixed dynamics over time \cite{chamberlain2021grand,gasteiger2019diffusion}. This prevents them from capturing node-specific behaviors and temporal variations in information flow. Therefore, these models struggle to adapt to dynamic, node-dependent diffusion patterns, limiting their effectiveness in complex real-world tasks.
    
    \item \textbf{Computational Complexity and Interpretability Issues:} Deeper GNNs are often employed to enhance expressiveness, but this comes at the cost of increased model parameters, particularly due to unshared weights, which significantly raises computational overhead \cite{li2021training,liu2020towards}. Additionally, deeper architectures obscure the flow of information, making it more difficult to interpret how predictions are formed.
    
    \item \textbf{Limited Understanding of Convergence Properties:}  While existing diffusion-based GNNs offer some insights into basic convergence dynamics, a rigorous theoretical characterization of their representational capacity remains lacking. In particular, the boundaries of the convergence configurations these models can capture are not well defined. Furthermore, the role of diffusion parameters in shaping convergence and representational behavior has not been systematically analyzed, hindering the adaptability of these models to diverse and complex tasks.
\end{itemize}

To address the aforementioned limitations, we propose incorporating principles from opinion dynamics into diffusion GNNs. Opinion dynamic models, such as the Friedkin-Johnsen~\cite{friedkin1990social}, are robust mathematical frameworks that are designed to capture how opinions propagate within a group of agents (i.e., nodes) in a social network \cite{das2014modeling,hassani2022classical}. These models trace dynamic behaviors of the underlying propagation, depicting how individual opinions evolve over time and eventually converge. We observe that these models exhibit unique characteristics that can significantly strengthen diffusion GNNs:

\begin{itemize}
    \item \textbf{Heterogeneous Diffusion with Temporal Evolution:} Opinion dynamics usually model node-specific (i.e., agent-specific) behaviors through parameters such as stubbornness (i.e., attachment to initial opinion) and capture temporal evolution through dynamic influence modeling between neighbors, enabling more complex propagations.

    \item \textbf{Efficiency and Interpretability:} Opinion dynamics provide interpretable, sequential information flow that can introduce a strong inductive bias. Unlike traditional GNNs that rely on heavy learnable transformations at each step, opinion dynamics achieve effective propagation through principled update rules with minimal parameterization, remaining simple and coherent even as model depth increases.

    \item \textbf{Theoretically Established Convergence Properties:} Opinion dynamics builds upon well-established mathematical foundations with provable convergence configurations \cite{degroot1974reaching,friedkin1990social,rainer2002opinion}, offering stronger theoretical insights into information propagation under diverse conditions. They also provide proper guidance on the impact of their propagation parameters, enabling adaptive parameter tuning based on task requirements.
\end{itemize}
 
However, integrating opinion dynamics into GNNs is a challenging task since it requires bridging several technical and conceptual gaps between computational social science and GNNs. This integration must effectively accommodate temporal modeling without incurring excessive computational costs and preserve convergence properties within the context of neural network training. Additionally, incorporating heterogeneous influence mechanisms requires adjustments to traditional message-passing schemes, which introduces new interpretability trade-offs that need careful consideration. Since different opinion dynamics models have distinct advantages, a unified framework is necessary to effectively integrate these mechanisms and fully leverage their potential to enhance GNN performance across various graph learning tasks. To address these challenges, we propose a novel \underline{G}eneralized \underline{O}pinion \underline{D}ynamics \underline{N}eural \underline{F}ramework, \emph{GODNF}, which is a more robust and theoretically grounded alternative to conventional diffusion GNNs. 

\paragraph{\textbf{Contributions}} Our main contributions are summarized as follows. (1) We introduce GODNF, a generalized framework grounded in multiple opinion dynamic models, which provides a principled foundation for designing diffusion-based GNNs and enables the efficient modeling of diverse information propagation behaviors.  (2) We advance diffusion GNNs by designing a message passing framework that captures evolving temporal dynamics with node-specific responses. (3) We show that our framework can accommodate deep layers while maintaining interpretability and computational efficiency. (4) We provide theoretical insights by proving that our framework guarantees convergence and can exhibit diverse convergence configurations, offering greater flexibility in modeling various diffusion behaviors. (5) Our extensive experiments on both static and temporal prediction tasks demonstrate that GODNF outperforms leading spatial and diffusion-based GNN models.

%% file: Sections/related_work.tex
\section{Related Work}

\paragraph{\textbf{Diffusion GNNs}}  Diffusion  GNNs are inspired by dynamical systems, where the propagation of information through the network layers is analogous to the time evolution of a dynamic process \cite{hancontinuous}. These GNNs are based on two primary diffusion paradigms: discrete and continuous. Discrete diffusion GNNs propagate information through iterative steps with explicit layer updates \cite{gasteiger2018predict,gasteiger2019diffusion,chienadaptive,zhao2021adaptive}. They are often computationally efficient and straightforward to implement, but struggle to capture fine-grained diffusion dynamics, as the use of fixed, uniform update rules limits their ability to model heterogeneous node behaviors and evolving temporal patterns. In contrast, continuous models frame diffusion as dynamical systems governed by differential equations \cite{chamberlain2021grand,chamberlain2021beltrami,choi2023gread,eliasof2024feature}. The adaptive nature of continuous diffusion modeling allows these GNNs to capture intricate diffusion patterns more effectively than their discrete counterparts.
Early continuous GNNs, such as GRAND \cite{chamberlain2021grand}, closely mimic heat diffusion but are prone to issues like oversmoothing and fail to capture long-range interactions in graphs effectively \cite{thorpegrand++}. Recent continuous diffusion GNNs have advanced beyond this by employing techniques such as additional damping/source terms \cite{thorpegrand++,choi2023gread,rusch2022graph,eliasof2024feature}, external forces \cite{wang2022acmp,zhao2023graph}, and geometric tools \cite{bodnar2022neural,chen2025graph} to mitigate these problems. Nonetheless, these GNNs require passing node features into differential equation solvers, which introduces high computational overhead and makes it difficult for them to scale to deep layers due to numerical stability issues.

Our framework combines the strengths of both discrete and continuous diffusion GNNs. Unlike standard discrete diffusion GNNs, which rely on uniform update rules, our opinion dynamic-based update rule captures intricate diffusion patterns through dynamic influence modeling and node-specific behavior adaptation, all while maintaining minimal learnable parameters. This provides our model with the adaptability of continuous models without their computational overhead and numerical stability issues. Furthermore, our method builds on well-established mathematical foundations with provable convergence properties, offering better real-world adaptability and interpretability compared to existing diffusion GNNs. As a unified framework, it captures a broader range of diffusion dynamics, resulting enhanced expressive power.

\paragraph{\textbf{Opinion Dynamics and GNNs}}

Concepts from opinion dynamic models have influenced the design of GNN architectures. GNNs with residual connections \cite{gasteiger2018predict,xu2018representation,liu2021graph,zhang2023drgcn} exhibit similarities to the Friedkin–Johnsen model \cite{friedkin1990social}, where the residual connections to the initial feature states mimic the concept of stubbornness in the model. Sheaf theory-based GNNs \cite{bodnar2022neural} adopt an opinion dynamics-inspired approach, where each node has both public and private opinions, and information propagation is modeled as the interplay between these two types of opinions. Zhou et al. \cite{zhouodnet}, and Vargas-P{\'e}rez et al. \cite{vargas2024unveiling}  introduced message passing mechanisms based on the bounded confidence opinion model \cite{rainer2002opinion}, designed to enhance information propagation in GNNs. Wang et al. \cite{wang2025resolving} proposed a message-passing approach inspired by behavior opinion dynamics \cite{leonard2024fast} to address the oversmoothing issue in GNNs. Compared to these works, our method provides a single unified framework that encapsulates diverse opinion dynamics behaviours. Recently, Li et al. \cite{li2025unigo} proposed UniGO, a GNN framework for temporal opinion evolution prediction. While UniGO offers a unified formulation of opinion dynamics, it employs this formulation solely to synthesise training datasets. The actual GNN model instead relies on a separate coarsen-refine mechanism with graph pooling. In contrast, our framework integrates unified opinion dynamics directly into the neural diffusion mechanism, yielding a principled mathematical foundation for message passing. This integration not only provides theoretical convergence guarantees but also ensures broader applicability beyond domain-specific prediction tasks.


%% file: Sections/background.tex
\section{Background}

Let $ G = (V, E) $ be an undirected graph, where $ V $ denotes the set of nodes, $ E $ the set of edges, $|V| = n$, and $|E| = m$. The nodes are represented by the feature matrix $ H \in \mathbb{R}^{n \times f} $, where $f$ denotes the number of features per node. Further, $ N(i) $ represents the set of nodes adjacent to node $i$.

\subsection{Diffusion Graph Neural Networks}

Diffusion GNNs extend message-passing by modeling information propagation as a diffusion process. Unlike local aggregation, they incorporate global structure through operators like the graph Laplacian or transition matrices \cite{gasteiger2019diffusion}.

A diffusion GNN is described by the evolution of node features over time through a temporal operator $\mathcal{D}$:

\[
\mathcal{D}[{X}(t)] = \Phi \bigg({X}(0), {X}(t), G; \Theta \bigg)
\]

where ${X}(t) \in \mathbb{R}^{n \times d}$ denotes the node features at time $ t $, with $ d $ being the feature dimensionality. The initial node features in ${X}(0)$ are typically derived from the raw node features in ${H}$ through a learnable transformation. $\Phi$ is a learnable diffusion operator parameterized by $\Theta$, which governs how information proagates across $G$.

\textbf{Discrete-Time Diffusion.}
In the discrete case, the operator takes the form $ \mathcal{D}[\mathbf{X}(t)] = \mathbf{X}(t+1) - \mathbf{X}(t) $, yielding the iterative update rule:
\[
{X}(t+1) = {X}(t) + \Phi \bigg({X}(0), {X}(t), G; \Theta \bigg)
\]

This formulation includes architectures such as APPNP \cite{gasteiger2018predict} and GDC \cite{gasteiger2019diffusion}.

\textbf{Continuous-Time Diffusion.}
Alternatively, in continuous-time models, the evolution is governed by an ordinary differential equation (ODE):
\[
\frac{d{X}(t)}{dt} = \Phi\bigg({X}(0), {X}(t), G; \Theta\bigg)
\]
Notable examples include GRAND \cite{chamberlain2021grand} and PDE-GCN \cite{eliasof2021pde}.

\subsection{Opinion Dynamics Models}

Opinion dynamics models describe how the states or opinions of nodes in a network evolve through local interactions, capturing different mechanisms of information propagation.

\textbf{ French-DeGroot (FD) Model \cite{degroot1974reaching}.}  
In the FD model, each node updates its opinion as a weighted average of its neighbors’ opinions:
\[
x_i(t+1) = \sum_{j=1}^{n} W_{ij} x_j(t)
\]
where $ x_i(t) $ denotes the opinion of node $ i $ at time $t$. The weight matrix $ W \in \mathbb{R}^{n \times n} $ is row-stochastic, with $ W_{ij} > 0 $ if node $i$ is influenced by node $j$. When the graph is connected, the opinions of all nodes converge to the same value (i.e., single consensus) \cite{degroot1974reaching}.

\textbf{Friedkin-Johnsen (FJ) Model \cite{friedkin1990social}.}  
The FJ model extends FD by allowing each node to retain partial attachment to its initial opinion:
\[
x_i(t+1) = \lambda_i x_i(0) + (1 - \lambda_i) \sum_{j=1}^{n} W_{ij} 
 x_j(t)
\]
 where  $x_i(0)$ is the initial opinion of node $i$, and $ \lambda_i \in [0,1] $ representing node $i$'s level of attachement to its initial opinion. This opinion dynamics converges to an equilibrium state that reflects a balance between the influence of neighboring nodes and each node’s adherence to its initial opinion, resulting in persistent diversity of opinions across the network \cite{biondi2023dynamics,xu2022effects}.

\textbf{Hegselmann–Krause (HK) Model \cite{rainer2002opinion}.}  
In the HK model, each node updates its opinion by averaging those of other nodes whose opinions lie within a specified confidence interval:
\begin{align*}
x_i(t+1) &= \frac{1}{|\Gamma(i,t)|} \sum_{j \in \Gamma(i,t)} x_j(t), \\
\Gamma(i,t) &= \{ j \in N(i) \mid |x_i(t) - x_j(t)| \leq \epsilon \}
\end{align*}

where $ \epsilon > 0 $ is the confidence bound. Unlike the FD and FJ models, the HK model depicts time-varying neighborhood influence as each node's influential neighbors change dynamically based on opinion proximity at each time step.  This opinion dynamics typically leads to multiple opinion clusters instead of a global consensus \cite{chen2017opinion}.

\begin{figure*}[]
  \centering  \includegraphics[width=0.9\textwidth, height=0.23\textheight]{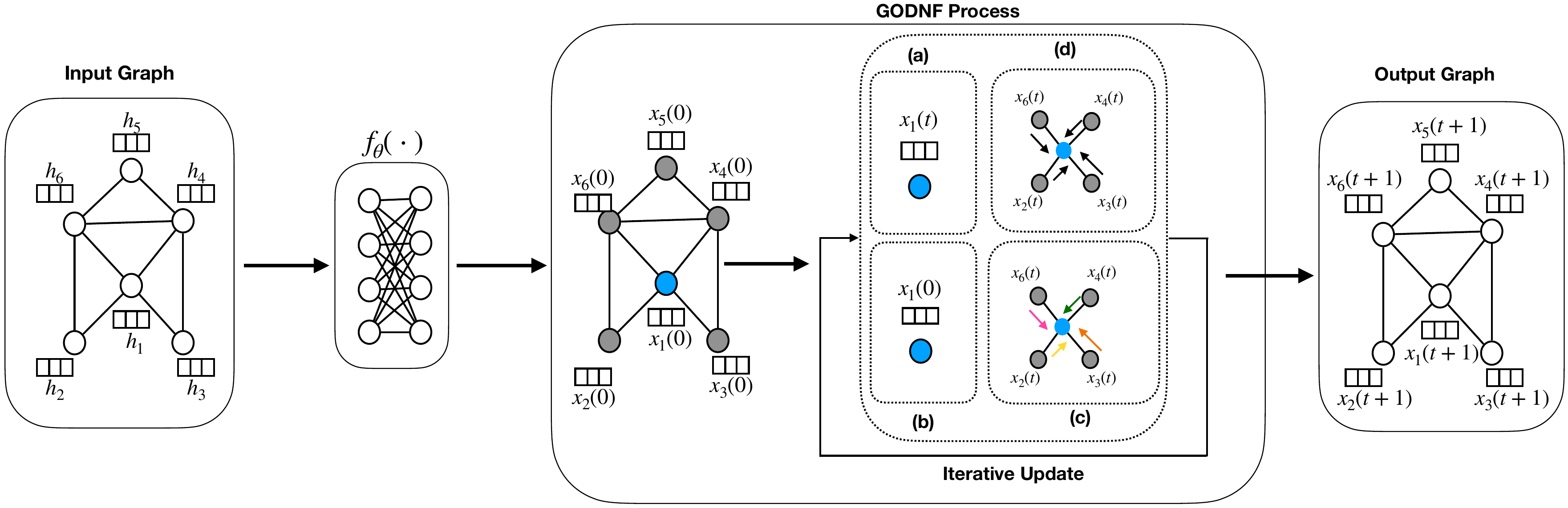}
  \caption{High-Level Architecture of GODNF: Raw features are first transformed to retrieve initial features. These initial features are iteratively updated based on four aspects: (a) Current feature retention, (b) Initial feature attachment, (c) Neighborhood influence, and (d) Structural regularization. The model is trained in an end-to-end manner.}
  \label{fig:model summary}
\end{figure*}


%% file: Sections/methodology.tex
\section{A Generalized Opinion Dynamics Framework for Graph Neural Diffusion}

In this section, we introduce a generalized discrete-time framework that leverages opinion dynamics for diffusion GNNs. We demonstrate how node features in diffusion GNNs evolve similarly to opinions in social networks, facilitating the transfer of insights across domains to develop a more powerful learning method for capturing complex diffusion patterns in graph-structured data. 

\subsection{ Diffusion Mechanism in GODNF}

Let $h_i \in H$ denote the feature vector of node $i$, and \( f_{\theta} \) be a neural network parameterized by \( \theta \). The update rule of GODNF is defined as follows: 
\[
x_i(0) = f_{\theta}(h_i),
\] 
\begin{equation}
\begin{aligned}
x_i(t+1) &= \alpha \underbrace{x_i(t)}_{\text{(a)}} + (1-\alpha)\bigg[ \lambda_i \underbrace{x_i(0)}_{\text{(b)}} + (1 - \lambda_i)\Big( \underbrace{\sum_{j \in \mathcal{N}(i)} w_{ij}(t) x_j(t)}_{\text{(c)}} \\
&\quad - \mu \underbrace{L_g x_i(t)}_{\text{(d)}} \Big) \bigg]
\end{aligned}
\label{eq:update_rule}
\end{equation}


The update dynamics for a node’s feature vector combine current feature retention, initial feature attachment, neighborhood influence, and structural regularization:

\begin{enumerate}[label=(\alph*)]
    \item \textbf{Currrent Feature Retention}: \( x_i(t) \) represents the node’s current feature vector. Controlled by a retention parameter $\alpha$ $(0 \leq \alpha < 1)$, this term preserves the node's recent state. 
    \item \textbf{Initial Feature Attachment}: \( x_i(0) \) represents the node’s initial feature after neural transformation, and the learnable parameter $\lambda_i$ represents the node's stubbornness. A higher $\lambda_i$ means node $i$ is more attached to its initial feature, showing greater resistance to external influence.
    
    \item \textbf{Neighborhood Influence}: This term aggregates features from node $i$'s neighbors. Each neighbor feature $x_j(t)$ is weighted by a time-dependent learnable parameter $w_{ij}(t)$ where $\sum_{j \in \mathcal{N}(i)}w_{ij}(t)=1$. These adaptive weights allow the model to adjust the influence of each neighbor on node $i$ over time, thereby capturing temporal evolutions.
    
    \item \textbf{Structural Regularization}: This term promotes feature smoothness across the nodes using a Laplacian operator $L_g$. The learnable coefficient \( \mu \in \mathbb{R}_{\ge 0} \) controls the regularization strength, encouraging similar features among connected nodes. 
\end{enumerate}

Together, these terms allow GODNF to dynamically balance individual stubbornness with network-level adaptation, enabling nuanced diffusion dynamics over time.

\paragraph{\textbf{Matrix Form}}
Let \( X(t) \in \mathbb{R}^{n \times d} \) denote the feature matrix at time \( t \), with \( X(0) \) as the initial feature matrix after neural transformation. Stubbornness is modeled by a diagonal matrix \( \Lambda = \operatorname{diag}(\lambda_1, \dots, \lambda_n) \). Neighbor influence is captured by the time-dependent row-stochastic influence matrix \( W(t) \in \mathbb{R}^{n \times n} \), where \( W_{ij}(t) \) denotes the influence of node \( j \) on node \( i \) at time step $t$. Structural regularization is enforced via the normalized Laplacian operator \( L_g \in \mathbb{R}^{n \times n} \).
With these notations, the update rule can be written in matrix form as:
\[
 X(0) = f_{\theta}(H)
\]
\[
X(t+1) = \alpha X(t) + (1-\alpha) \left[ \Lambda X(0) + (I - \Lambda) \left( W(t) - \mu L_g \right) X(t) \right].
\]

\paragraph{\textbf{A Simplified Variant: GODNF$_{\text{Static}}$}}

We also introduce a simplified variant of our model where the influence matrix \( W(t) \) is constant over time (\( W(t) = W \) for all \( t \)), meaning the influence between nodes remains fixed throughout propagation. The update rule then simplifies to:
\[
X(t+1) = \alpha X(t) + (1-\alpha) \left[ \Lambda X(0) + (I - \Lambda) \left( W - \mu L_g \right)X(t) \right].
\]

For brevity, we refer to our original model with the time-dependent influence matrix as \textbf{GODNF$_{\text{Dynamic}}$}. A detailed mapping of GODNF components to existing opinion dynamic models is provided in the Appendix \ref{sec:mapping}.

\subsection{Bounded Weight Evolution}

We bound the time-dependent evolution of influence weights in the GODNF$_{\text{Dynamic}}$ model to prevent oscillation of $W(t)$, thereby ensuring stability in model training. To achieve this, we employ a decaying update rule:

\[
W(t+1) = W(t) + \eta(t) \Delta W(t)
\]

where \( \eta(t) \) is a time-dependent learning rate that decreases over time with $\lim_{t \to \infty} \eta(t) = 0$. 
$\Delta W(t)$ is the weight adjustment computed from the gradient of the objective function w.r.t. $W(t)$. This ensures progressively smaller updates and eventual stabilization of \( W(t) \) to a fixed matrix $W^*$ as \( t \) increases. 


\begin{remark}
 $\lim_{t \to \infty} \eta(t) = 0$ implies the convergence of  $W(t)$ into $W^*$ under following conditions: (1) $ ||\Delta W(t)|| \leq k $ for all $t$ and some constant $k > 0$; (2) $\sum_{t=0}^{\infty} \eta(t) = \infty$ and $\sum_{t=0}^{\infty} \eta(t)^2 < \infty$ \cite{robbins1951stochastic}. Condition (1) is naturally satisfied since $\Delta W(t)$ is computed on gradient weights under a well-behaved training loss function. To satisfy condition (2), we select a simple and efficient function $\eta(t) = \frac{1}{1+t}$.
\end{remark}





\subsection{Convergence Conditions of GODNF}

We present the following theorem regarding the convergence of GODNF. It demonstrates that, under given conditions, the node features will converge to a unique fixed point. The proof is provided in the Appendix.


\begin{restatable}[Convergence of GODNF Feature Updates]{theorem}{convergence} 

Let \( X(t) \in \mathbb{R}^{n \times d} \) evolve according to the update rule in Eq.~\eqref{eq:update_rule}, and define the combined influence matrix \( M(t) = (I - \Lambda)(W(t) - \mu L_g) \). When the parameters \( \Lambda \) and \( \mu \) are chosen such that the operator norm satisfies \( \|M(t)\|_{\mathrm{op}} < 1 \) for all \( t \), and \( W(t) \to W^* \), the \( X(t) \) is guaranteed to converge to a unique fixed point \( X^* \) as \( t \to \infty \).
\label{thm:convergence}
\end{restatable}

\begin{remark}
The condition $\|M(t)\|_{\mathrm{op}} < 1$ is realistic and not overly restrictive as $W(t)$ (row-stochastic with eigenvalues in $[-1,1]$) and $L_g$ (normalized Laplacian with eigenvalues in $[0,2]$) can be appropriately scaled by learnable parameters $\Lambda$ and $\mu$ to satisfy this bound regardless of graph structure.\end{remark}

\subsection{Convergence-Preserving Regularization}

According to Theorem~\ref{thm:convergence}, convergence of GODNF requires \(\|M(t)\|_{\mathrm{op}} \) < 1 at every time step. Therefore, we design a regularization term enforcing this condition. However, computing \( \|M(t)\|_{\mathrm{op}} \) can be expensive in practice. To address this, we employ an efficient approximation based on a well-established matrix norm inequality. The operator norm of a matrix is bounded above by the geometric mean of its 1-norm and infinity-norm \cite{horn2012matrix},
\[
\|M(t)\|_{\mathrm{op}} \leq \sqrt{\|M(t)\|_1 \cdot \|M(t)\|_\infty},
\]
 Unlike the $\mathrm{op}$-norm, the infinity-norm and 1-norm can be computed in linear time.

\paragraph{\textbf{Regularization Term}} We design a regularization term to enforce the condition \( \sqrt{\|M(t)\|_1 \cdot \|M(t)\|_\infty} < 1 \), penalizing violations when this condition is not satisfied. The regularization term is:

\begin{equation*}
\label{eq:reg}
\mathcal{L}_{reg} = \sum_{t=0}^{T} \max\left(0, \sqrt{\|M(t)\|_1 \cdot \|M(t)\|_\infty} - 1 \right).
\end{equation*}

In GODNF$_\text{static}$ variant, since \( W(t) = W \) is constant, which makes \( M(t) = M \) time invariant, the regularization term can be simplified to:

\[
\mathcal{L}_{reg} = \max\left(0, \sqrt{\|M\|_1 \cdot \|M\|_\infty} - 1 \right).
\]
This eliminates the need to sum over time steps, reducing computational complexity and making GODNF$_\text{static}$ more efficient than its dynamic counterpart.


\paragraph{\textbf{Training Loss}} We define training loss of GODNF as a combination of downstream task loss and the regularization term as follows:
\[
\mathcal{L}_{total} = \mathcal{L}_{task} +  \mathcal{L}_{reg}.
\]

%% file: Sections/theoretical_analysis.tex
\section{Theoretical Analysis}

\subsection{Complexity Analysis}
 
The time complexity of GODNF comprises the initialization and iterative update phases. Initial feature tranformation costs \(\mathcal{O}(n \cdot p_f)\), where \(p_f\) is the upper bound on the cost of computing \(f_{\theta}(h_i)\). The complexity of each iteration incurred by GODNF components is: initial feature attachment $\mathcal{O}(n \cdot d)$, neighborhood influence and structural regularization both $\mathcal{O}(m \cdot d)$ through sparse matrix multiplication, and norm regularization term $\mathcal{O}(m)$, which is efficient due to the sparsity of $M(t)$. Here, $d$ is the feature dimension. The total time complexity over \(T\) iterations is \(\mathcal{O}(n \cdot p_f + T \cdot (m \cdot d + n \cdot d + m)) = \mathcal{O}(n \cdot p_f + T \cdot m \cdot d)\), using the graph connectivity assumption. In practical scenarios where $ p_f$, $d$, and $T$ are usually small, this complexity is bounded by $\mathcal{O}(m)$, which is asymptotically optimal.

The space complexity of GODNF includes storing feature matrices \(X(0)\) and \(X(t)\) with \(\mathcal{O}(n \cdot d)\), the sparse matrices \(W(t)\) and \(L_g\)  with \(\mathcal{O}(m)\), and the diagonal matrix \(\Lambda\) with \(\mathcal{O}(n)\). Overall, the space complexity is \(\mathcal{O}(n \cdot d + m)\).

Overall, GODNF achieves comparable asymptotic complexity to traditional naive GNNs such as GCN \cite{kipf2017semi}. The absence of intermediate feature transformations in GODNF can lead to computational efficiency gains, especially for deeper layers. An empirical analysis of the scalability of our approach is provided in Section \ref{sec:scalability_analysis}.

\subsection{Representation of Diverse Convergence Configurations}

\label{sec:consensus_states}

In this section, we demonstrate that GODNF can capture multiple convergence configurations, emphasizing its adaptability to diverse learning scenarios. We start by defining these convergence configurations and show that GODNF can theoretically converge to these diverse configurations by adjusting its components. For the following theorems, we consider the following conditions of GODNF to hold: (1) $W^t \rightarrow W^*$, (2) \( \|M(t)\|_{\mathrm{op}} < 1 \) for all \( t \) where \( M(t) = (I - \Lambda)(W(t) - \mu L_g) \). All the proofs are provided in the Appendix.

\subsubsection{Single Consensus}

\begin{definition}
A diffusion model is said to reach a \emph{single consensus} if there exists a vector \( v \in \mathbb{R}^d \) such that, for all nodes \( i \in V \),
\[
\lim_{t \to \infty} x_i(t) = v,
\]
\end{definition}

We now establish conditions under which GODNF converges to a single consensus. 

\begin{restatable}[Single Consensus under Fully Diffusive Dynamics]{theorem}{singleconsensus} 
GODNF reaches a single consensus if
 \( \Lambda = 0 \), \( \mu = 0 \), \( \alpha = 0 \), and  $W^*$ is a row-stochastic matrix.

\end{restatable}

Single consensus is valuable for modeling linear diffusion processes on networks, such as heat diffusion and basic opinion dynamics, as it captures the convergence of these processes.

\begin{table*}[!htbp]
\centering
\scalebox{0.825}{
\begin{tabular}{l@{\hspace{6pt}}c@{\hspace{6pt}}c@{\hspace{6pt}}c@{\hspace{6pt}}c@{\hspace{6pt}}c@{\hspace{6pt}}c@{\hspace{6pt}}c@{\hspace{6pt}}c@{\hspace{6pt}}c@{\hspace{6pt}}c}
\toprule
& Texas & Cornell & Wisconsin & Film & Amazon-rating & Cora Full & Citeseer & Cora-ML & PubMed & DBLP \\
\midrule
GCN & 60.00{\scriptsize$\pm$6.45} & 55.14{\scriptsize$\pm$8.46} & 61.60{\scriptsize$\pm$7.00} & 32.16{\scriptsize$\pm$1.34} & 37.99{\scriptsize$\pm$0.61} & 68.06{\scriptsize$\pm$0.98} & 77.72{\scriptsize$\pm$1.18} & 87.07{\scriptsize$\pm$1.21} & 86.74{\scriptsize$\pm$0.47} & 83.93{\scriptsize$\pm$0.84} \\
GAT & 61.21{\scriptsize$\pm$8.17} & 53.64{\scriptsize$\pm$11.10} & 60.00{\scriptsize$\pm$11.00} & 32.63{\scriptsize$\pm$1.61} & 42.52{\scriptsize$\pm$1.22} & 67.55{\scriptsize$\pm$1.23} & 76.79{\scriptsize$\pm$1.00} & 84.12{\scriptsize$\pm$0.55} & 87.24{\scriptsize$\pm$0.55} & 80.61{\scriptsize$\pm$1.21} \\
GraphSAGE & 82.46{\scriptsize$\pm$5.19} & 75.32{\scriptsize$\pm$5.96} & 80.00{\scriptsize$\pm$4.81} & 36.00{\scriptsize$\pm$1.20} & 41.32{\scriptsize$\pm$0.63} & 69.77{\scriptsize$\pm$0.46} & 77.42{\scriptsize$\pm$0.83} & 86.52{\scriptsize$\pm$1.32} & 84.50{\scriptsize$\pm$0.39} & 86.16{\scriptsize$\pm$0.50} \\
AirGNN & 65.57{\scriptsize$\pm$11.19} & 53.19{\scriptsize$\pm$10.43} & 56.25{\scriptsize$\pm$5.03} & 36.34{\scriptsize$\pm$2.73} &  42.29{\scriptsize$\pm$0.88} & 65.27{\scriptsize$\pm$0.72} &  78.20{\scriptsize$\pm$1.17} &  88.09{\scriptsize$\pm$1.25} & 83.51{\scriptsize$\pm$0.87} &  84.69{\scriptsize$\pm$0.30} \\
DirGNN & 76.25{\scriptsize$\pm$6.31} & 76.51{\scriptsize$\pm$6.14} & 80.50{\scriptsize$\pm$5.50} & 35.76{\scriptsize$\pm$1.68} & 46.66{\scriptsize$\pm$0.61} & 67.80{\scriptsize$\pm$0.53} & 77.71{\scriptsize$\pm$0.78} & 85.66{\scriptsize$\pm$0.31} & 86.94{\scriptsize$\pm$0.55} & 81.22{\scriptsize$\pm$0.54} \\
BEC-GCN & {68.85\scriptsize$\pm$11.02} & {65.96\scriptsize$\pm$9.03} & {66.25\scriptsize$\pm$4.55} & {34.34\scriptsize$\pm$2.08} & {42.17\scriptsize$\pm$0.58} & {71.34\scriptsize$\pm$0.44} & {78.14\scriptsize$\pm$1.31} & {88.65\scriptsize$\pm$1.32} & {87.00\scriptsize$\pm$0.28} & {85.20\scriptsize$\pm$0.48} \\
\midrule
APPNP & 62.30{\scriptsize$\pm$6.65} & 58.75{\scriptsize$\pm$5.97} & 61.70{\scriptsize$\pm$6.94} & 36.44{\scriptsize$\pm$1.78} & 43.48{\scriptsize$\pm$0.68} & 70.63{\scriptsize$\pm$0.51} & 77.89{\scriptsize$\pm$1.07} & 87.41{\scriptsize$\pm$1.54} & 83.37{\scriptsize$\pm$0.42} & 85.36{\scriptsize$\pm$0.40} \\
GDC & 72.13{\scriptsize$\pm$5.48} & 53.19{\scriptsize$\pm$12.19} & 62.50{\scriptsize$\pm$5.13} & 30.65{\scriptsize$\pm$1.20} & 40.06{\scriptsize$\pm$2.50} & 62.23{\scriptsize$\pm$0.59} & 78.17{\scriptsize$\pm$1.11} & 84.41{\scriptsize$\pm$1.59} & 80.61{\scriptsize$\pm$0.47} & 84.87{\scriptsize$\pm$0.40} \\
ODNet & 85.25{\scriptsize$\pm$8.36} & 86.49{\scriptsize$\pm$4.06} & 88.75{\scriptsize$\pm$2.15} & 37.57{\scriptsize$\pm$0.94} & 43.33{\scriptsize$\pm$0.91} & 70.23{\scriptsize$\pm$0.43} & 74.76{\scriptsize$\pm$1.15} & 76.26{\scriptsize$\pm$1.03} & 86.74{\scriptsize$\pm$0.37} & 83.45{\scriptsize$\pm$0.47} \\
GRAND & 81.70{\scriptsize$\pm$8.42} & 81.76{\scriptsize$\pm$13.90} & 84.00{\scriptsize$\pm$7.50} & 27.93{\scriptsize$\pm$1.25} & 37.53{\scriptsize$\pm$0.36} & 67.66{\scriptsize$\pm$1.01} & 76.55{\scriptsize$\pm$1.69} & 88.49{\scriptsize$\pm$0.81} & 86.79{\scriptsize$\pm$0.57} & 84.60{\scriptsize$\pm$0.99} \\
GRAND++ & 79.34{\scriptsize$\pm$7.22} & 81.34{\scriptsize$\pm$7.12} & 81.50{\scriptsize$\pm$6.00} & 36.32{\scriptsize$\pm$0.50} & 38.01{\scriptsize$\pm$0.50} & 67.53{\scriptsize$\pm$0.74} & 78.51{\scriptsize$\pm$1.58} & 88.44{\scriptsize$\pm$0.53} & 87.21{\scriptsize$\pm$0.33} & 85.21{\scriptsize$\pm$0.24} \\
ACMP & 87.65{\scriptsize$\pm$3.54} & 85.66{\scriptsize$\pm$5.10} & 86.50{\scriptsize$\pm$5.00} & 34.44{\scriptsize$\pm$1.36} & 37.32{\scriptsize$\pm$0.64} & 71.76{\scriptsize$\pm$0.03} & 75.73{\scriptsize$\pm$1.38} & 76.11{\scriptsize$\pm$2.12} & 88.01{\scriptsize$\pm$1.44} & 82.31{\scriptsize$\pm$0.44} \\
HiD-Net & 77.11{\scriptsize$\pm$6.91} & 83.33{\scriptsize$\pm$7.10} & 81.30{\scriptsize$\pm$6.60} & 28.86{\scriptsize$\pm$1.00} & 41.19{\scriptsize$\pm$1.03} & 68.11{\scriptsize$\pm$0.64} & 77.85{\scriptsize$\pm$0.65} & 89.00{\scriptsize$\pm$0.51} & 88.60{\scriptsize$\pm$0.45} & 84.92{\scriptsize$\pm$0.31} \\
GNRF & 87.39{\scriptsize$\pm$4.13} & \cellcolor{orange!30}87.28{\scriptsize$\pm$3.12} & 88.00{\scriptsize$\pm$2.00} & 34.22{\scriptsize$\pm$1.40} & 46.89{\scriptsize$\pm$1.08} & 72.12{\scriptsize$\pm$0.50} & 75.79{\scriptsize$\pm$0.94} & \cellcolor{orange!30}89.18{\scriptsize$\pm$0.19} & 90.37{\scriptsize$\pm$0.69} & 85.73{\scriptsize$\pm$0.76} \\
\midrule
GODNF$_\text{Static}$ & \cellcolor{orange!30}92.95{\scriptsize$\pm$2.32} & 87.23{\scriptsize$\pm$4.66} & \cellcolor{orange!30}91.62{\scriptsize$\pm$3.49} & \cellcolor{blue!25}{40.21{\scriptsize$\pm$1.85}} & \cellcolor{orange!30}49.17{\scriptsize$\pm$1.06} & \cellcolor{blue!25}{72.60{\scriptsize$\pm$0.58}} & \cellcolor{orange!30}78.83{\scriptsize$\pm$1.22} & 88.85{\scriptsize$\pm$1.52} & \cellcolor{orange!30}90.66{\scriptsize$\pm$0.47} & \cellcolor{orange!30}88.24{\scriptsize$\pm$2.01} \\
GODNF$_\text{Dynamic}$ & \cellcolor{blue!25}{93.93{\scriptsize$\pm$2.65}} & \cellcolor{blue!25}{87.87{\scriptsize$\pm$4.57}} & \cellcolor{blue!25}{92.50{\scriptsize$\pm$2.56}} & \cellcolor{orange!30}39.96{\scriptsize$\pm$2.21} & \cellcolor{blue!25}{49.40{\scriptsize$\pm$0.44}} & \cellcolor{orange!30}72.51{\scriptsize$\pm$0.53} & \cellcolor{blue!25}{78.85{\scriptsize$\pm$0.99}} & \cellcolor{blue!25}{89.26{\scriptsize$\pm$1.13}} & \cellcolor{blue!25}{90.78{\scriptsize$\pm$0.68}} & \cellcolor{blue!25}{88.62{\scriptsize$\pm$1.92}} \\
\bottomrule
\end{tabular}
}
\caption{Node classification accuracy ± standard deviation (\%). The best and second-best results are highlighted in \textcolor{blue} {blue}, and \textcolor{orange}{orange}, respectively. The baseline results are sourced from 
 Chen et al. \cite{chen2025graph}.
}
\label{tab:node_classification_results}
\end{table*}

\subsubsection{Multi Consensus}

\begin{definition}
A diffusion model is said to reach a \emph{multi consensus} if there exist distinct vectors \( v_1, v_2, \dots, v_k \in \mathbb{R}^d \) with $k < n$ such that, for all nodes \( i \in V \), 
\[
\lim_{t \to \infty} x_i(t) = v_j, \quad j \in \{1, 2, \dots, k\}.
\]
\end{definition}

We now present the conditions under which GODNF converges to a multi consensus.

\begin{restatable}[Multi-Consensus under Block-Structured Dynamics]{theorem}{multiconsensus} 
GODNF reaches multi consensus if the matrix $M^* = (1 - \Lambda) (W^* - \mu L_g)$  has a block diagonal structure corresponding to multiple disconnected components, with initial conditions and parameters differ for at least two components.

\end{restatable}

The multi consensus is beneficial for multi-label classification as it can avoid the oversmoothing problem in traditional GNNs. GODNF can preserve multiple clusters rather than converging to a single value, effectively aligning node features with class boundaries and handling both homophilic and heterophilic scenarios.

\subsubsection{Individualized Consensus}

\begin{definition}
A diffusion model is said to reach an \emph{individualized consensus} if there exist distinct vectors \( v_1, v_2, \dots, v_n \in \mathbb{R}^d \) such that, for all nodes \( i \in V \),
\[
\lim_{t \to \infty} x_i(t) = v_i, \quad i \in \{1, 2, \dots, n\}, \text{ with } v_i \neq v_j \text{ for } i \neq j.
\]
\end{definition}

We next look into the conditions that lead GODNF to reach individualized consensus, where each node converges to a unique value.

\begin{restatable}[Individualized Consensus under Strong Feature Retention]{theorem}{individualizedconsensus} 
GODNF reaches individualized consensus if the initial features \( x_i(0) \) are distinct for all nodes \( i \in V \), and \( \lambda_i \in (0, 1] \) are sufficiently large (i.e., close to 1).
\end{restatable}

Individualized consensus is beneficial for node-level prediction tasks like regression. Unlike standard GNNs that oversmooth features, this consensus maintains distinct node features while still capturing network structure, offering adaptability to diverse graph settings where preserving node identity is essential.

\begin{remark}
In addition to the three fundamental convergence configurations above, GODNF can adapt to more complex convergence patterns due to its unification of diverse opinion dynamics mechanisms. This flexibility enables GODNF to align with sophisticated convergence configurations that capture the nuances present in the ground truth.
\end{remark}

%% file: Sections/experiments.tex
\section{Experiments}

We evaluate GODNF on two types of tasks: node classification and influence estimation. The goal of the node classification task is to predict the label of
individual nodes in a graph. The influence estimation is a node regression task that focuses on determining a node's susceptibility to activation given a specific diffusion model and an initial set of activated nodes \cite{xia2021deepis}. Although node classification operates as a static downstream task, influence estimation requires capturing the temporal dynamics of the diffusion process to learn the underlying propagation characteristics accurately.

\subsection{Experimental Setup}

\subsubsection{Datasets} For node classification, we use ten benchmark datasets that include both homophily and heterophily label patterns. For homophily datasets, we use the Cora-ML, CiteSeer, PubMed, DBLP, and Cora Full datasets from the CitationFull benchmark \cite{bojchevski2018deep}. For heterophily datasets, we include Texas, Cornell, Wisconsin, and Film datasets from WebKB \cite{pei2020geom}, and the Amazon-rating dataset from the Heterophilous Graph benchmark \cite{platonovcritical}. Further, we employ the OGB benchmark's ogbn-arxiv dataset \cite{hu2020open} for our scalability analysis. For influence estimation, we use four real-world datasets: Jazz, Network Science, Power Grid \cite{rossi2015network}, and Cora-ML \cite{bojchevski2018deep}.

\begin{table*}[!htbp]
\centering
\scalebox{0.775}{
\begin{tabular}{l|ccc|ccc|ccc|ccc}
\hline
\multirow{2}{*}{} & \multicolumn{3}{c|}{Jazz} & \multicolumn{3}{c|}{Cora-ML} & \multicolumn{3}{c|}{Network Science} & \multicolumn{3}{c}{Power Grid} \\
\cline{2-13}
 & IC & LT & SIS & IC & LT & SIS & IC & LT & SIS & IC & LT & SIS \\
\hline
GCN & 0.233{\scriptsize$\pm$0.010} & 0.199{\scriptsize$\pm$0.006} & 0.344{\scriptsize$\pm$0.023} & 0.277{\scriptsize$\pm$0.007} & 0.255{\scriptsize$\pm$0.008} & 0.365{\scriptsize$\pm$0.065}  & 0.270{\scriptsize$\pm$0.019} & 0.190{\scriptsize$\pm$0.012} & 0.180{\scriptsize$\pm$0.007}  & 0.313{\scriptsize$\pm$0.024} & 0.335{\scriptsize$\pm$0.023} & 0.207{\scriptsize$\pm$0.001} \\
GAT & 0.342{\scriptsize$\pm$0.005} & 0.156{\scriptsize$\pm$0.100} & \cellcolor{blue!25}{0.288{\scriptsize$\pm$0.017}} & 0.352{\scriptsize$\pm$0.004} & 0.192{\scriptsize$\pm$0.010} & 0.208{\scriptsize$\pm$0.008}  & 0.274{\scriptsize$\pm$0.002} & 0.114{\scriptsize$\pm$0.008} & 0.123{\scriptsize$\pm$0.013}  & 0.331{\scriptsize$\pm$0.002} & 0.280{\scriptsize$\pm$0.015} & 0.133{\scriptsize$\pm$0.001}  \\
GraphSAGE & 0.201{\scriptsize$\pm$0.028} & 0.120{\scriptsize$\pm$0.004} & \cellcolor{orange!30}0.301{\scriptsize$\pm$0.018}  & 0.255{\scriptsize$\pm$0.010} & 0.203{\scriptsize$\pm$0.019} & 0.222{\scriptsize$\pm$0.051}  & 0.241{\scriptsize$\pm$0.010} & 0.112{\scriptsize$\pm$0.005} & 0.102{\scriptsize$\pm$0.005}  & 0.313{\scriptsize$\pm$0.024} & 0.341{\scriptsize$\pm$0.018} & 0.133{\scriptsize$\pm$0.001}  \\
DirGNN & {0.205\scriptsize$\pm$0.016} & {0.124\scriptsize$\pm$0.002} & {0.379\scriptsize$\pm$0.032} & 0.255{\scriptsize$\pm$0.006} & {0.193\scriptsize$\pm$0.012} & 0.229{\scriptsize$\pm$0.025} & 0.236{\scriptsize$\pm$0.011} & 0.084{\scriptsize$\pm$0.017} & 0.142{\scriptsize$\pm$0.019} & 0.287{\scriptsize$\pm$0.010} & 0.257{\scriptsize$\pm$0.062} & 0.132{\scriptsize$\pm$0.001} \\
SGNN & {0.183\scriptsize$\pm$0.004} & {0.164\scriptsize$\pm$0.014} & {0.330\scriptsize$\pm$0.007} & {0.210\scriptsize$\pm$0.003} & {0.134\scriptsize$\pm$0.004} & {0.211\scriptsize$\pm$0.006} & {0.213\scriptsize$\pm$0.003} & {0.049\scriptsize$\pm$0.004} & 0.127{\scriptsize$\pm$0.00} &0.257{\scriptsize$\pm$0.002} & 0.192{\scriptsize$\pm$0.001} & 0.175{\scriptsize$\pm$0.004} \\
UniGO & {0.192\scriptsize$\pm$0.013} & {0.159\scriptsize$\pm$0.020} & {0.335\scriptsize$\pm$0.002} & {0.255\scriptsize$\pm$0.001} & {0.155\scriptsize$\pm$0.019} & {0.212\scriptsize$\pm$0.001} & {0.201\scriptsize$\pm$0.001} & {0.110\scriptsize$\pm$0.049} & {0.108\scriptsize$\pm$0.023} & {0.231\scriptsize$\pm$0.002} & {0.235\scriptsize$\pm$0.005} & {0.206\scriptsize$\pm$0.023} \\
 \hline
 DeepIS & \cellcolor{orange!30}0.151{\scriptsize$\pm$0.003} & 0.219{\scriptsize$\pm$0.002} & 0.434{\scriptsize$\pm$0.003} & 0.203{\scriptsize$\pm$0.001}  & 0.301{\scriptsize$\pm$0.005} & 0.304{\scriptsize$\pm$0.001} & 0.223{\scriptsize$\pm$0.001} & 0.306{\scriptsize$\pm$0.001} & 0.256{\scriptsize$\pm$0.001} & \cellcolor{orange!30}0.206{\scriptsize$\pm$0.001} & 0.374{\scriptsize$\pm$0.001} & 0.251{\scriptsize$\pm$0.001} \\
 DeepIM & 0.178{\scriptsize$\pm$0.002} & 0.134{\scriptsize$\pm$0.014} & 0.383{\scriptsize$\pm$0.010} & 0.210{\scriptsize$\pm$0.002} & 0.271{\scriptsize$\pm$0.010} & 0.270{\scriptsize$\pm$0.006} & 0.216{\scriptsize$\pm$0.003} & 0.118{\scriptsize$\pm$0.003} & 0.135{\scriptsize$\pm$0.004} & 0.258{\scriptsize$\pm$0.003} & 0.331{\scriptsize$\pm$0.002} & 0.205{\scriptsize$\pm$0.005} \\
GLIE & \cellcolor{blue!25}{0.136{\scriptsize$\pm$0.003}} & 0.055{\scriptsize$\pm$0.028} & 0.454{\scriptsize$\pm$0.062} & 0.199{\scriptsize$\pm$0.041} & 0.286{\scriptsize$\pm$0.016} & 0.205{\scriptsize$\pm$0.029}  & \cellcolor{blue!25}{0.163{\scriptsize$\pm$0.031}} & 0.160{\scriptsize$\pm$0.063} & {0.103\scriptsize$\pm$0.023} &
\cellcolor{blue!25}{0.183{\scriptsize$\pm$0.009}} & 0.384{\scriptsize$\pm$0.020}  & 0.132{\scriptsize$\pm$0.026} \\
APPNP & 0.200{\scriptsize$\pm$0.006} & 0.124{\scriptsize$\pm$0.003} & 0.357{\scriptsize$\pm$0.059} & 0.265{\scriptsize$\pm$0.022} & 0.220{\scriptsize$\pm$0.031} & 0.321{\scriptsize$\pm$0.042} & 0.248{\scriptsize$\pm$0.006} & 0.084{\scriptsize$\pm$0.006} & 0.100{\scriptsize$\pm$0.012} & 0.290{\scriptsize$\pm$0.006} & 
0.189{\scriptsize$\pm$0.011} & 0.132{\scriptsize$\pm$0.001} \\
ODNet & 0.180{\scriptsize$\pm$0.003} & 0.053{\scriptsize$\pm$0.003} & 0.322{\scriptsize$\pm$0.006} & 0.232{\scriptsize$\pm$0.001} & 0.210{\scriptsize$\pm$0.004} & \cellcolor{orange!30}0.196{\scriptsize$\pm$0.002} & 0.255{\scriptsize$\pm$0.002} & 0.104{\scriptsize$\pm$0.015} & 0.106{\scriptsize$\pm$0.005} & 0.274{\scriptsize$\pm$0.001} & 0.296{\scriptsize$\pm$0.002} & 0.145{\scriptsize$\pm$0.001} \\
HIDNet & 0.216{\scriptsize$\pm$0.003} & 0.180{\scriptsize$\pm$0.003} & 0.404{\scriptsize$\pm$0.005} & 0.258{\scriptsize$\pm$0.001} & 0.273{\scriptsize$\pm$0.006} & 0.360{\scriptsize$\pm$0.003} & 0.242{\scriptsize$\pm$0.002} & 0.166{\scriptsize$\pm$0.005} & 0.168{\scriptsize$\pm$0.001} & 0.282{\scriptsize$\pm$0.001} & 0.295{\scriptsize$\pm$0.007} & 0.224{\scriptsize$\pm$0.001} \\
GNRF & 0.195{\scriptsize$\pm$0.006} & 0.150{\scriptsize$\pm$0.020} & 0.444{\scriptsize$\pm$0.024} & 0.255{\scriptsize$\pm$0.005} & 0.312{\scriptsize$\pm$0.029} & 0.280{\scriptsize$\pm$0.027} & 0.294{\scriptsize$\pm$0.005} & 0.293{\scriptsize$\pm$0.001} & 0.186{\scriptsize$\pm$0.006} & 0.280{\scriptsize$\pm$0.003} & 0.459{\scriptsize$\pm$0.019} & 0.196{\scriptsize$\pm$0.018} \\
\hline
GODNF$_\text{Static}$ & 0.177{\scriptsize$\pm$0.005} & \cellcolor{orange!30}0.051{\scriptsize$\pm$0.002} & 0.318{\scriptsize$\pm$0.004} & \cellcolor{orange!30}0.150{\scriptsize$\pm$0.001} & \cellcolor{orange!30}0.104{\scriptsize$\pm$0.005} & 0.197{\scriptsize$\pm$0.001} & 0.205{\scriptsize$\pm$0.002} & \cellcolor{orange!30}0.036{\scriptsize$\pm$0.002} & \cellcolor{orange!30}0.098{\scriptsize$\pm$0.002} & 0.223{\scriptsize$\pm$0.002} & \cellcolor{orange!30}0.116{\scriptsize$\pm$0.003} & \cellcolor{orange!30}0.120{\scriptsize$\pm$0.002} \\
GODNF$_\text{Dynamic}$ & 0.175{\scriptsize$\pm$0.003} & \cellcolor{blue!25}{0.050{\scriptsize$\pm$0.001}} & 0.309{\scriptsize$\pm$0.006} & \cellcolor{blue!25}{0.148{\scriptsize$\pm$0.001}} & \cellcolor{blue!25}{0.094{\scriptsize$\pm$0.005}} & \cellcolor{blue!25}{0.192{\scriptsize$\pm$0.002}} & \cellcolor{orange!30}0.197{\scriptsize$\pm$0.002} & \cellcolor{blue!25}{0.032{\scriptsize$\pm$0.003}} & \cellcolor{blue!25}{0.096{\scriptsize$\pm$0.001}} & 0.221{\scriptsize$\pm$0.001} & \cellcolor{blue!25}{0.108{\scriptsize$\pm$0.004}} & \cellcolor{blue!25}{0.119{\scriptsize$\pm$0.002}} \\
\hline
\end{tabular}
}
\caption{Influence estimation mean absolute error ± standard deviation. Lower value indicates better performance. The best and second-best results are highlighted in \textcolor{blue}{blue}, and \textcolor{orange}{orange}, respectively. }
\label{tab:influence_estimation_results}
\end{table*}

\subsubsection{Baselines}
For node classification, we compare GODNF against both spatial and diffusion-based GNNs. Our spatial GNN baselines consist of three classical architectures: GCN \cite{kipf2017semi}, GAT \cite{velivckovic2018graph}, and GraphSAGE \cite{hamilton2017inductive}, as well as three advanced models: AirGNN \cite{liu2021graph}, DirGNN \cite{rossi2024edge}, and BEC-GCN \cite{hevapathige2025depth}. Diffusion GNNs include both discrete and continuous variants, spanning classical and state-of-the-art models such as APPNP \cite{gasteiger2018predict}, GDC \cite{gasteiger2019diffusion},  ODNET \cite{zhouodnet}, GRAND \cite{chamberlain2021grand}, GRAND++ \cite{thorpegrand++}, ACMP \cite{wang2022acmp}, HiD-Net \cite{li2024generalized}, and GNRF \cite{chen2025graph}.  AirGNN and APPNP have residual connections in their architectures that are conceptually similar to the FJ model, while ODNet is designed based on the HK model. For influence estimation, we use a comprehensive set of aforementioned general-purpose baselines that span spatial and diffusion GNNs and GNNs specifically designed for the influence estimation task. These specialized models include SGNN \cite{kumar2022influence}, DeepIM \cite{ling2023deep}, DeepIS \cite{xia2021deepis}, and GLIE \cite{panagopoulos2023maximizing}. Additionally, we tailor UniGO \cite{li2025unigo} as a baseline for influence estimation.

\subsubsection{Evaluation settings}

For node classification, we use a 60:20:20 random split strategy for the training, validation, and test sets and report the mean and standard deviation of accuracy over 10 random initializations, following the experimental setup of Chen et al. \cite{chen2025graph}. We evaluate influence estimation performance under three diffusion models: Linear Threshold (LT) \cite{kempe2003maximizing}, Independent Cascade (IC) \cite{goldenberg2001talk}, and Susceptible-Infected-Susceptible (SIS) \cite{kimura2009efficient}. LT and IC are progressive diffusion models (i.e., nodes activate once and remain activated permanently), whereas SIS is non-progressive (i.e., nodes can undergo multiple activation cycles) \cite{li2023influence}. For data splits, we follow the experimental setup of Ling et al. \cite{ling2023deep}, with an initial activation node set consisting of 10\% of the total nodes in the graph. For evaluation, we employ 10-fold cross-validation and report the mean and standard deviation of Mean Absolute Error (MAE). For baseline results, we report the results from previous papers that used the same experimental setup. When such results are unavailable, we generate baseline results by adopting the hyperparameter configurations specified in the original papers.

Additional experimental details, including model hyperparameters, implementation details, and benchmark dataset statistics, are provided in the Appendix \ref{sec:experimental_details}.

\subsection{Main Results}

\subsubsection{Node Classification}

We present node classification results in Table \ref{tab:node_classification_results}. GODNF consistently outperforms baseline models in both heterophily and homophily datasets. In particular, GODNF demonstrates significant superiority over spatial and discrete diffusion GNNs, especially within heterophily datasets. Additionally, GODNF comfortably surpasses continuous diffusion GNNs by exceeding their expressive power. We credit GODNF's success to its capacity to model complex diffusion dynamics in graphs, which is enhanced by its representation power derived from multiple opinion dynamic components. The performance of models such as ODNET, APPNP, and AirGNN further supports this point. Although these models incorporate elements related to opinion dynamics, they do not achieve the same level of performance as GODNF, primarily because their components lack the representation capabilities found in GODNF. Further, we observe that the dynamic variant of GODNF performs better than the static variant because it captures temporal dependencies between diffusion time steps. 

\subsubsection{Influence Estimation}

We report influence estimation results in Table \ref{tab:influence_estimation_results}. GODNF outperforms or provides comparable performance to both general-purpose spatial and diffusion GNNs, as well as GNNs specifically designed for influence estimation. While DeepIS and GLIE excel in the IC model due to their tailored update rules, they exhibit suboptimal performance within other diffusion models. In contrast, GODNF shows robust generalizability, consistently maintaining better performance across all diffusion models. Additionally, GODNF$_\text{dynamic}$ consistently outperforms GODNF$_\text{static}$ because the temporal dynamics captured by the dynamic variant are beneficial for tasks like influence estimation, which rely on temporal characteristics.

\begin{figure*}
    \centering
    \includegraphics[width=\linewidth]{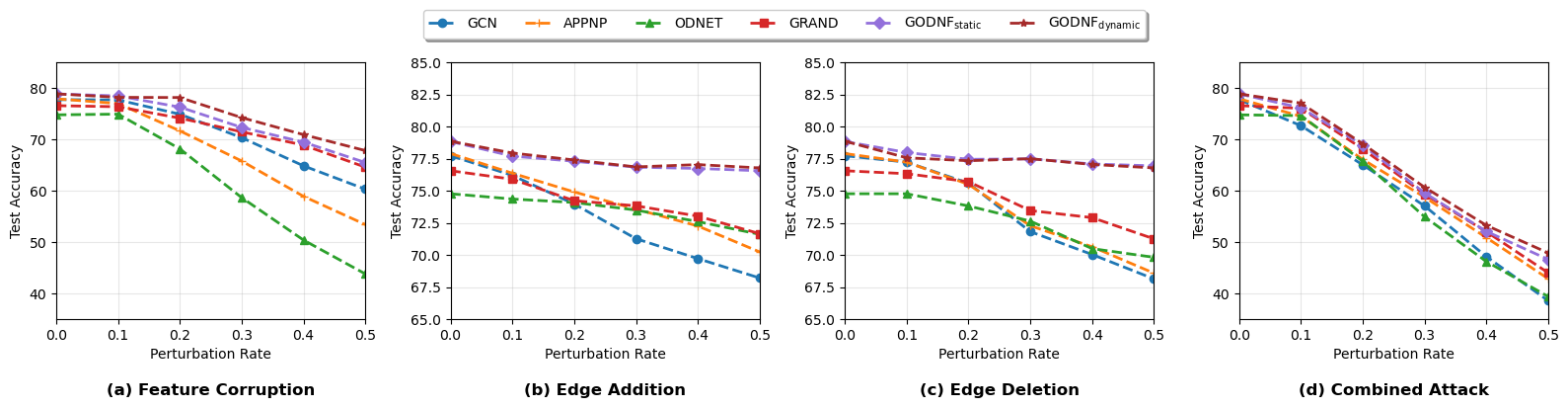}
    \caption{Node classification performance for Citeseer dataset under different adversarial attacks}
    \label{fig:attack_benchmark}
\end{figure*}

\subsubsection{Robustness to adversarial attacks}

We evaluate the robustness of GODNF under a diverse set of adversarial attack scenarios, benchmarking its resilience against four distinct attack strategies and comparing it with established baselines. In the feature corruption setting, Gaussian noise is injected into node features by sampling a noise matrix and adding it to the original features, following Li et al. \cite{li2024generalized}. The edge addition attack perturbs the topology by inserting random edges (excluding self-loops and duplicates) at a specified perturbation rate, while edge deletion removes randomly selected edges. The combined attack applies edge deletion, edge addition, and feature corruption sequentially at the same perturbation rate. Results for the Citeseer dataset are presented in Figure \ref{fig:attack_benchmark}, with additional results for other datasets provided in the Appendix \ref{sec:adversarial_attack}.

GODNF exhibits consistent and superior robustness across all adversarial attack types. Under feature corruption, it shows only gradual performance degradation versus steep drops in baselines. For structural attacks such as edge addition or deletion, GODNF variants maintain near-constant performance while baselines degrade significantly. Combined attacks further accentuate GODNF's resilience. This robustness stems from GODNF's multi-component defence mechanism, creating redundant information pathways. The initial feature attachment acts as persistent memory, structural regularization provides graph-based smoothing, and bounded weight evolution ensures stability. When attacks compromise one component, the others compensate, thereby preventing the single-pathway failures common in standard GNNs. 

\subsubsection{Oversmoothing Analysis}

We evaluate the performance of GODNF in a deep layer setting to understand its resilience to oversmoothing. The results are depicted in Figure \ref{fig:oversmoothing_comparison}. Unlike traditional GNNs like GCN and GRAND, which show substantial performance degradation when the layer depth is increased, both GODNF variants show robustness against oversmoothing by maintaining consistent performance across varying layer depths.

\begin{figure}[htbp]
    \centering
    \includegraphics[width=0.475\textwidth]{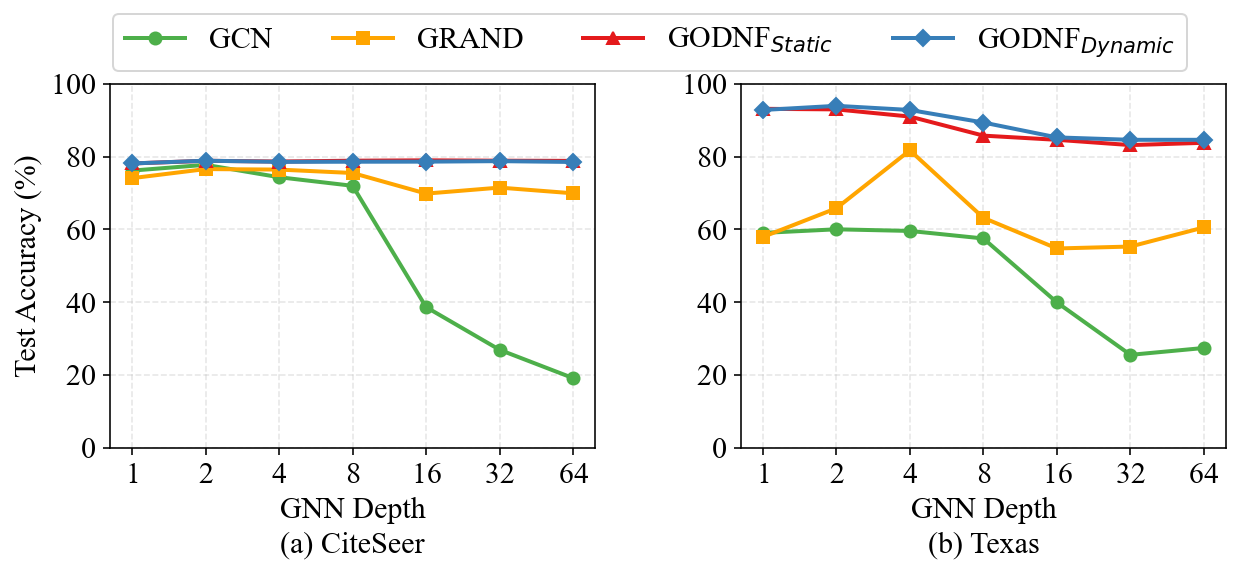}
    \caption{Deep layer comparison for node classification}
    \label{fig:oversmoothing_comparison}
\end{figure}

\subsection{Ablation Studies}

\subsubsection{Impact of Different Components}

We conduct an ablation study to evaluate the importance of each component of GODNF on its performance. To demonstrate this, we derive variants
of GODNF by omitting each component and assess their performance. The results are depicted in Figure \ref{fig:component_analysis}.

\begin{figure}[htbp]
    \centering
    \includegraphics[width=0.475\textwidth]{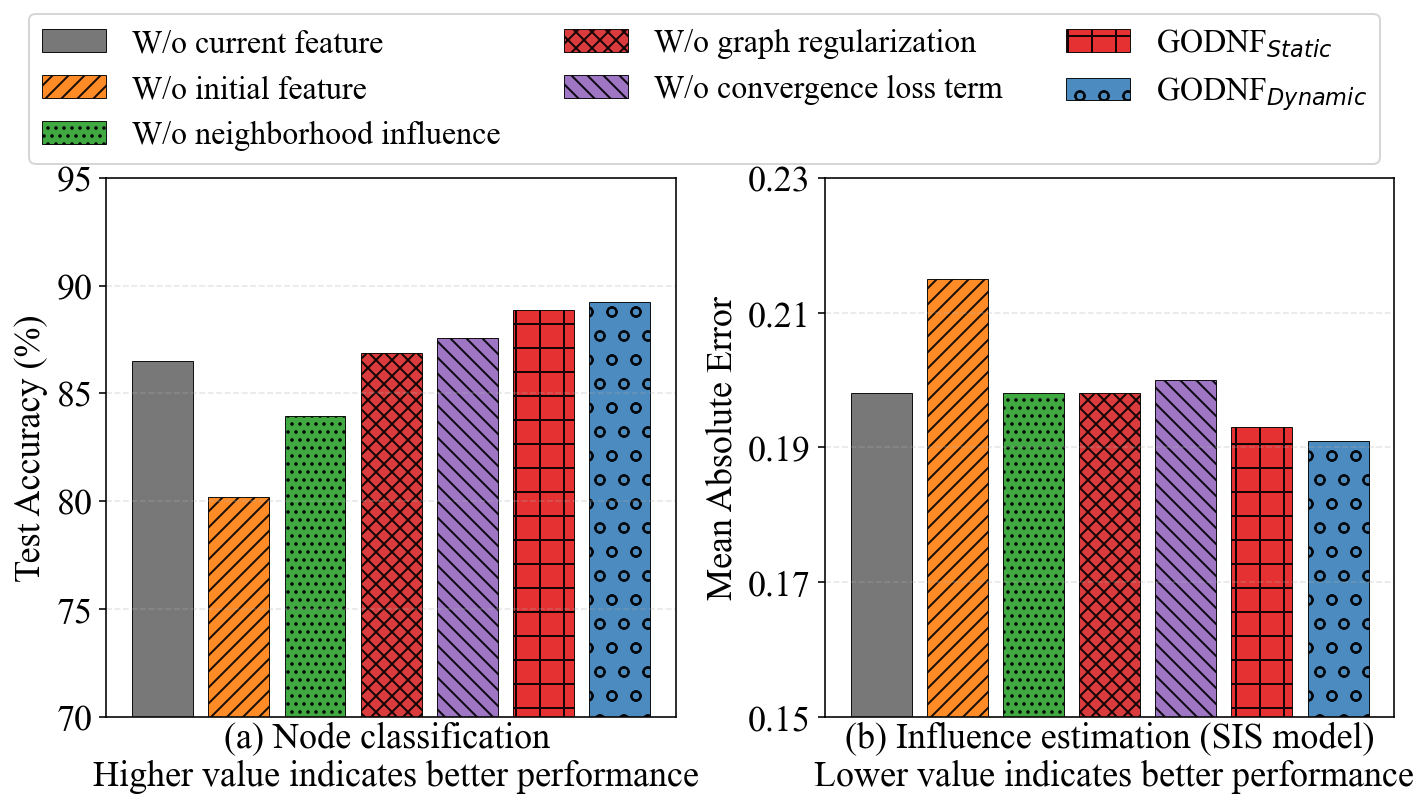}
    \caption{Performance comparison of GODNF component variants on Cora ML dataset}
    \label{fig:component_analysis}
\end{figure}

As anticipated, removing each component decreases performance, demonstrating their individual importance. Further, we observe that the initial feature attachment has more impact on performance than other components. We believe this is due to complex node-specific behaviors captured by that component.

\subsubsection{Parameter Analysis}

We evaluate the sensitivity of GODNF to its hyperparameter $\alpha$ in Eq. \eqref{eq:update_rule}, by varying its value from 0 to 0.9. In Figure \ref{fig:parameter_analysis}, we observe that the performance of homophilic datasets (Cora ML and Coral Full) in node classification tends to decrease slowly as $\alpha$ increases, while the opposite behavior can be seen in heterophilic datasets (Film and Texas). This is because a decrease of $\alpha$ promotes neighborhood information aggregation, which helps homophily datasets where neighbors tend to share the same label, but adversely impacts heterophilic datasets, where neighbors tend to have different labels.

\begin{figure}[htbp]
    \centering
    \includegraphics[width=0.475\textwidth]{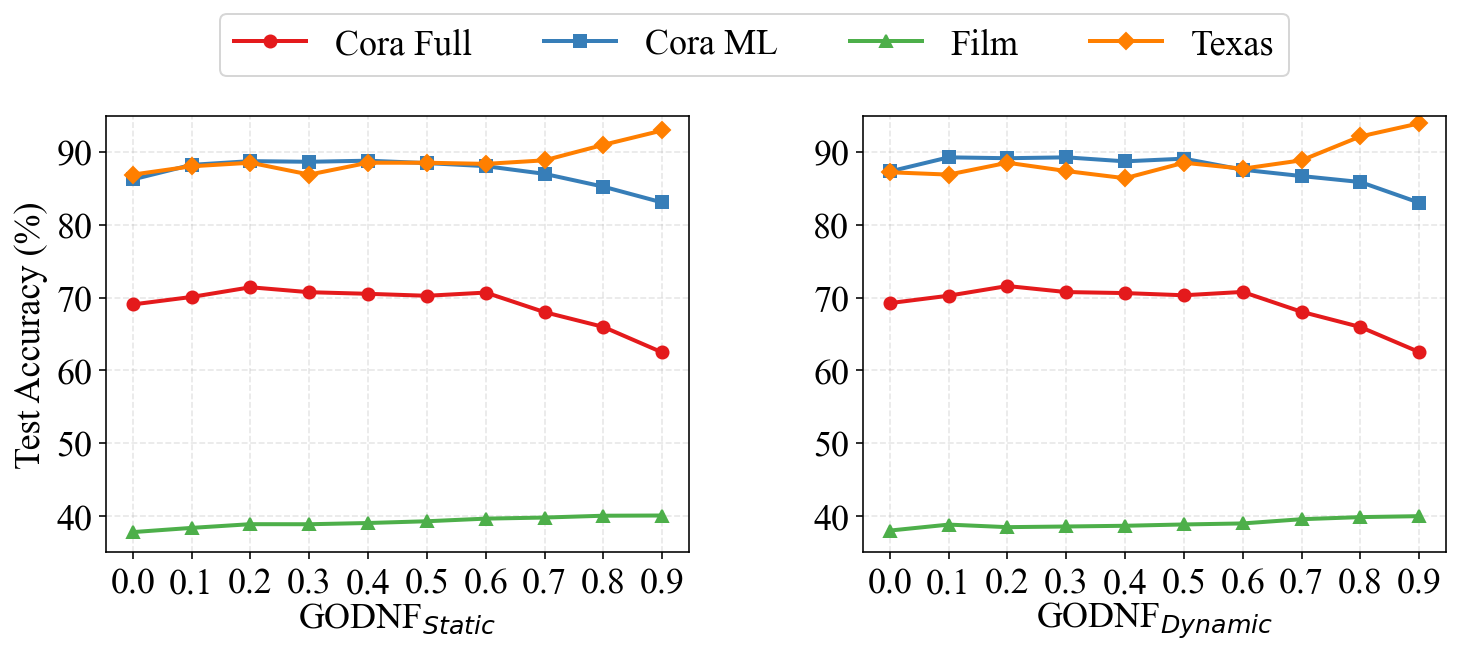}
    \caption{Sensitivity analysis of $\alpha$}
    \label{fig:parameter_analysis}
\end{figure}

Then, we present the distribution of learned node stubbernness $\lambda$ for datasets under node classification in Figure \ref{fig:stubbernness_values}. The results reveal few interesting observations. First, most nodes exhibit moderate stubbornness (0.4-0.7), highlighting the importance of balancing initial features with neighborhood information. Second, the distribution variability shows adaptive learning of heterogeneous aggregation requirements across nodes. Third, heterophily datasets tend to yield higher stubbornness values compared to homophily datasets, suggesting that when neighbors have dissimilar labels, nodes benefit from retaining their initial representations to mitigate information loss from inappropriate aggregation.

We further observe that the learned value for $\mu$ is in the range of 0.7-0.9 for most datasets. We believe this range offers essential structural regularisation for convergence stability.


\begin{figure}[htbp]
    \centering
    \scalebox{0.6}{
    \begin{tabular}{cc}
        \includegraphics[width=0.35\textwidth]{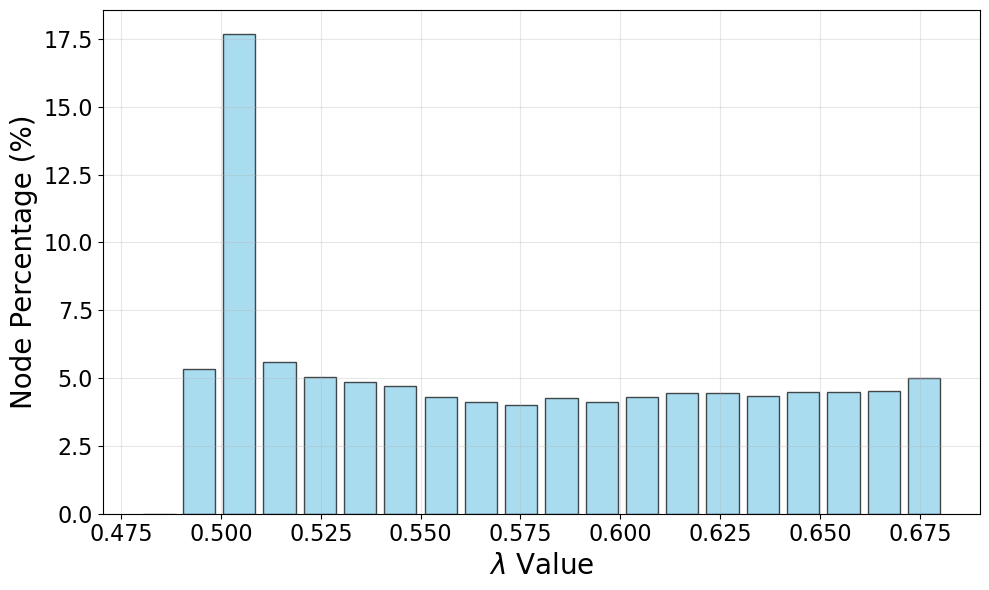}  &
        \includegraphics[width=0.35\textwidth]{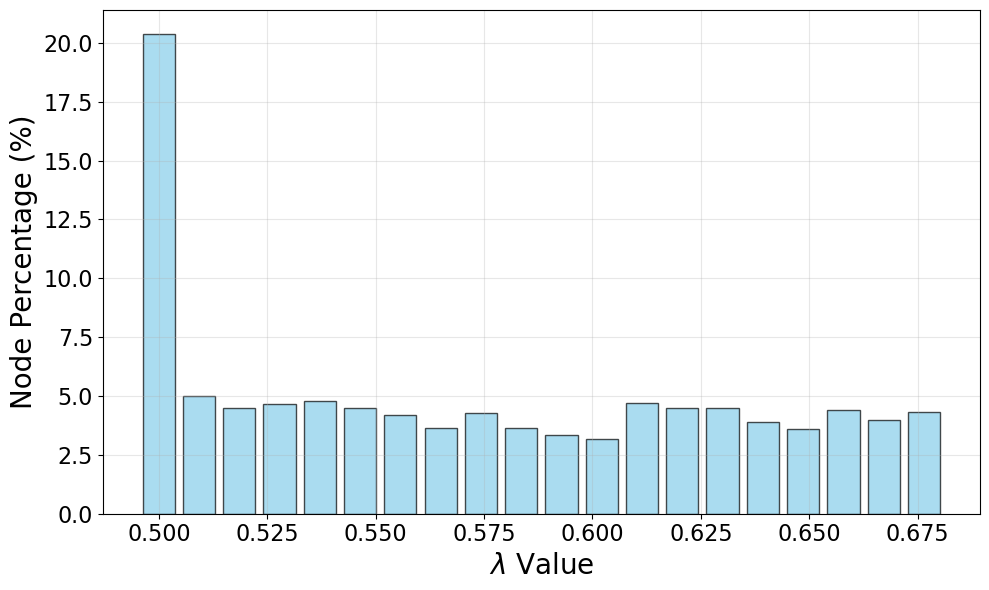} \\
        (a) Cora Full & (b) Cora ML \\
        \includegraphics[width=0.35\textwidth]{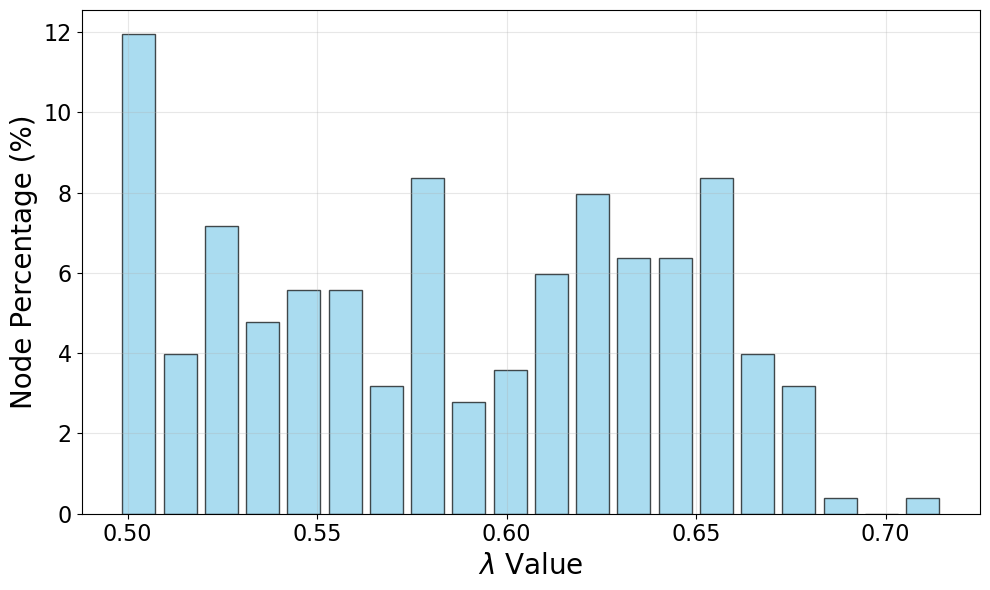} &
        \includegraphics[width=0.35\textwidth]{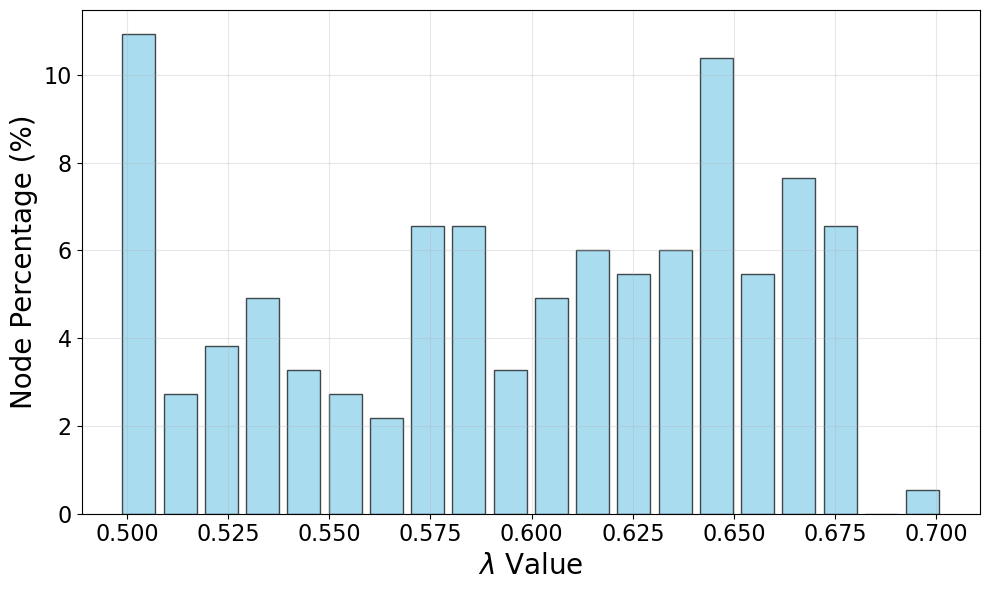} \\
        (c) Film & (d) Texas
    \end{tabular}
    }
    \caption{Distribution of learned node stubbornness $\lambda_i$} 
    \label{fig:stubbernness_values}
\end{figure}


\subsubsection{Scalability Analysis}
\label{sec:scalability_analysis}

We compare the scalability of our approach with existing GNNs in Table \ref{tab:scalability_comparison}. For this, we perform node classification task and employ two datasets: Cora Full and ogbn-arxiv. ogbn-arxiv is a large-scale dataset with 169,343 nodes and 1,166,243 edges. For a fair comparison, we fix the hidden layer size at 128 for all GNNs, and record average runtime per epoch, number of learnable parameters, and accuracy metrics.

\begin{table}[htbp]
\centering
\scalebox{0.7}{
\begin{tabular}{l|l|c|c|c|c|c|c}
\hline
\multirow{3}{*}{Layers} & \multirow{3}{*}{Model} & \multicolumn{3}{c|}{Cora Full} & \multicolumn{3}{c}{ogbn-arxiv} \\ \cline{3-8} 
 &  & \# Param & Time & Acc. & \# Param & Time & Acc. \\ 
 &  & (k) & (ms) & (\%) & (k) & (ms) & (\%) \\ \hline
\multirow{6}{*}{4} & GCN & 1157 & 20.6 & 68.44 & 55 & 154.2 & 69.53 \\ \cline{2-8} 
 & GAT & 4631 & 65.3 & 59.87 & 206 & OOM & OOM \\ \cline{2-8} 
 & GRAND & 716 & 111.1 & 69.27 & 81 & 2077.0 & 69.30 \\ \cline{2-8} 
 & ACMP & 1158 & 1107.6 & 70.94 & 55 & OOM & OOM \\ \cline{2-8} 
 & GODNF$_\text{Static}$ & 1144 & 31.7 & 72.42 & 191 & 926.4 & 69.76 \\ \cline{2-8} 
 & GODNF$_\text{Dynamic}$ & 1144 & 35.5 & \cellcolor{blue!25}{72.58} & 191 & 936.9 & \cellcolor{blue!25}{70.43} \\ \hline
\multirow{6}{*}{8} & GCN & 1223 & 34.6 & 65.87 & 123 & 327.2 & 68.39 \\ \cline{2-8} 
 & GAT & 4897 & 121.3 & 31.10 & 474 & OOM & OOM \\ \cline{2-8} 
 & GRAND & 716 & 150.7 & 69.81 & 81 & 3196.5 & 67.24 \\ \cline{2-8} 
 & ACMP & 1158 & 2231.5 & 71.76 & 55 & OOM & OOM \\ \cline{2-8} 
 & GODNF$_\text{Static}$ & 1144 & 51.1 & 72.42 & 191 & 2156.1 & 70.23 \\ \cline{2-8} 
 & GODNF$_\text{Dynamic}$ & 1144 & 63.4 & \cellcolor{blue!25}{72.49} & 191 & 2378.4 & \cellcolor{blue!25}{70.91} \\ \hline
\end{tabular}
}
\caption{Scalability comparison for node classification. Best results are highlighted.}
\label{tab:scalability_comparison}
\end{table}

As observed, GODNF is significantly more effective than traditional GNNs like GAT and diffusion GNNs like GRAND and ACMP. Furthermore, GODNF has a runtime and parameter complexity comparable to that of highly scalable traditional GNNs like GCN. This efficiency primarily comes from eliminating the learnable transformation layers found in conventional GNNs, resulting in less parameter count and memory usage, even at increasing depths.

\subsubsection{Case Study}

We empirically validate GODNF's ability to depict multiple convergence configurations. Specifically, we create a stochastic block model (SBM) graph \cite{abbe2018community} (50 nodes and five communities) and generate node labels to be aligned with different convergence configurations, as described in Section \ref{sec:consensus_states}, and train our model to learn these label patterns. In a multi-consensus scenario, we examine homophily and heterophily label patterns separately.  The visualization of color-coded node labels learned by our model for each convergence configuration is depicted in Figure \ref{fig:consensus_states}.

\begin{figure}[htbp]
    \centering
    \scalebox{0.65}{
    \begin{tabular}{cc}
        \includegraphics[width=0.35\textwidth]{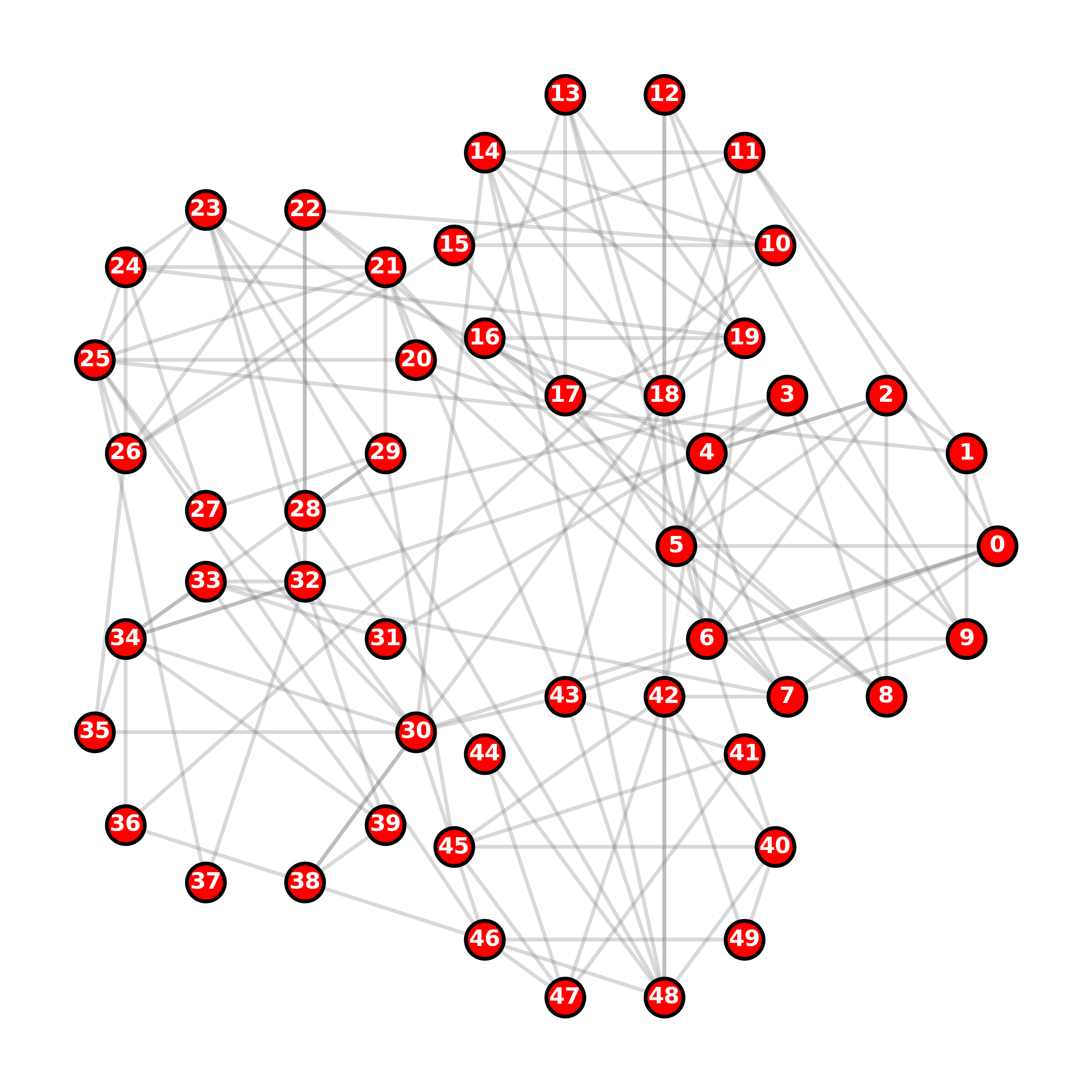} \label{fig:5a} &
        \includegraphics[width=0.35\textwidth]{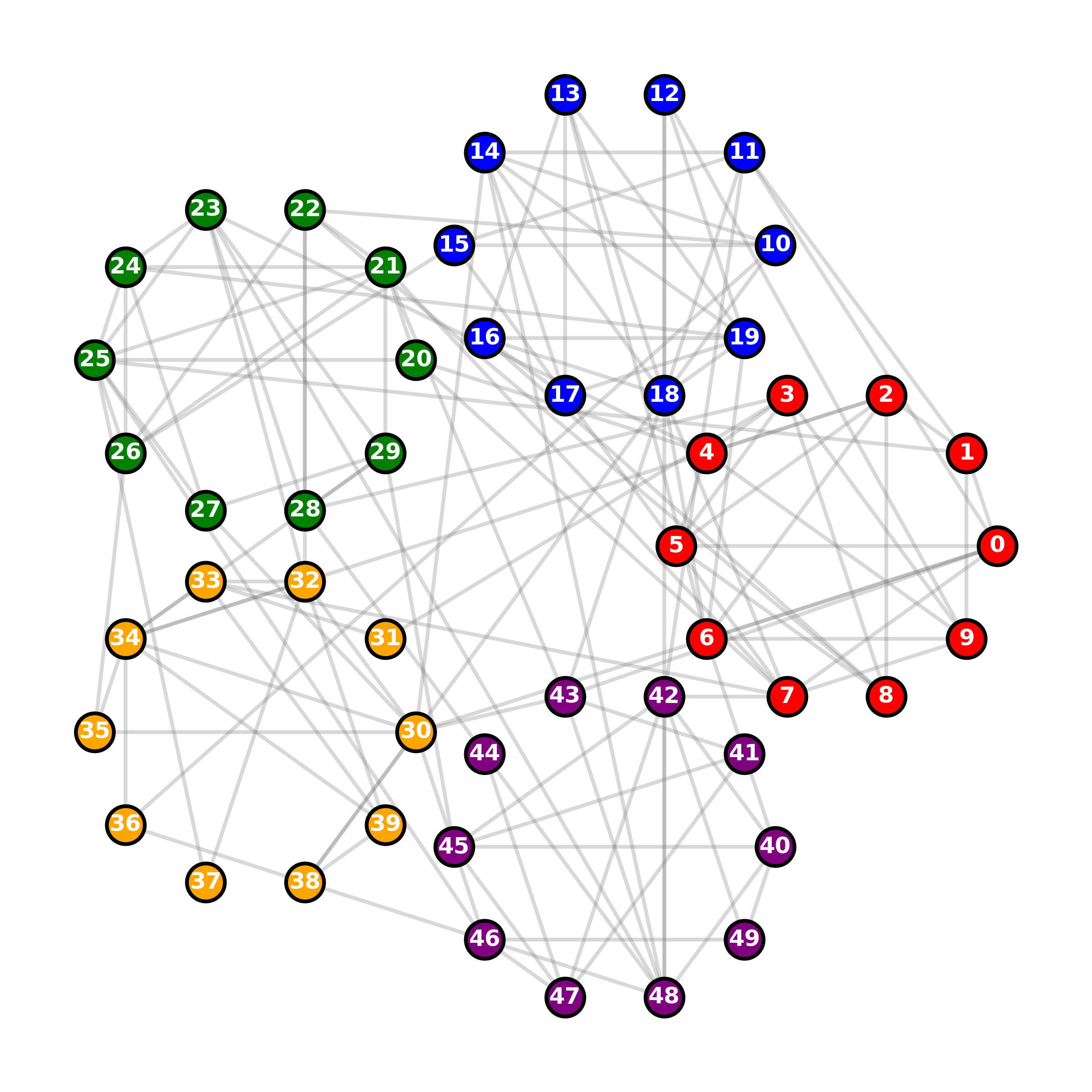}  \label{fig:5b} \\
        (a) Single Consensus & (b) Multi Consensus - Homophily Pattern \\
        \includegraphics[width=0.35\textwidth]{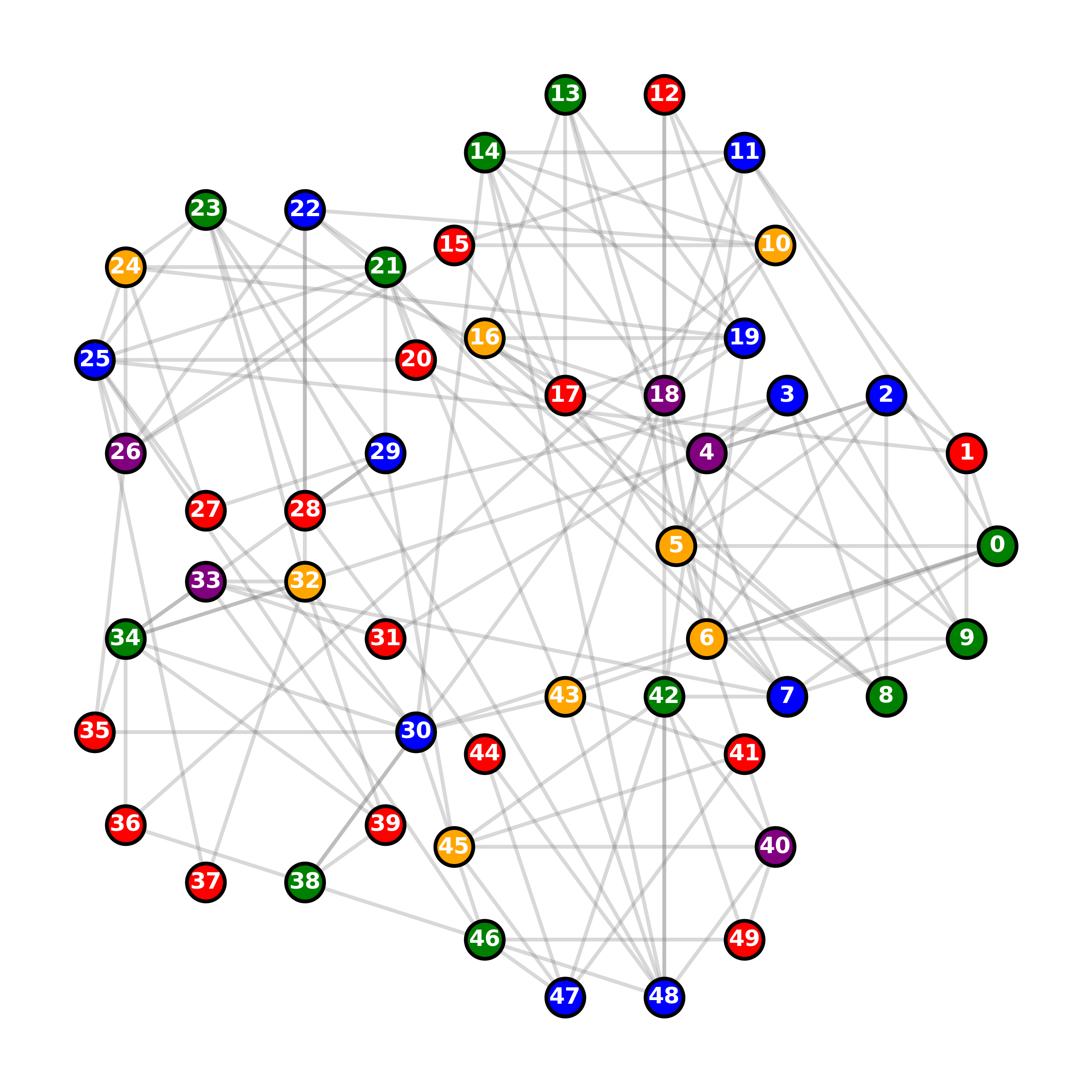}  \label{fig:5c} &
        \includegraphics[width=0.35\textwidth]{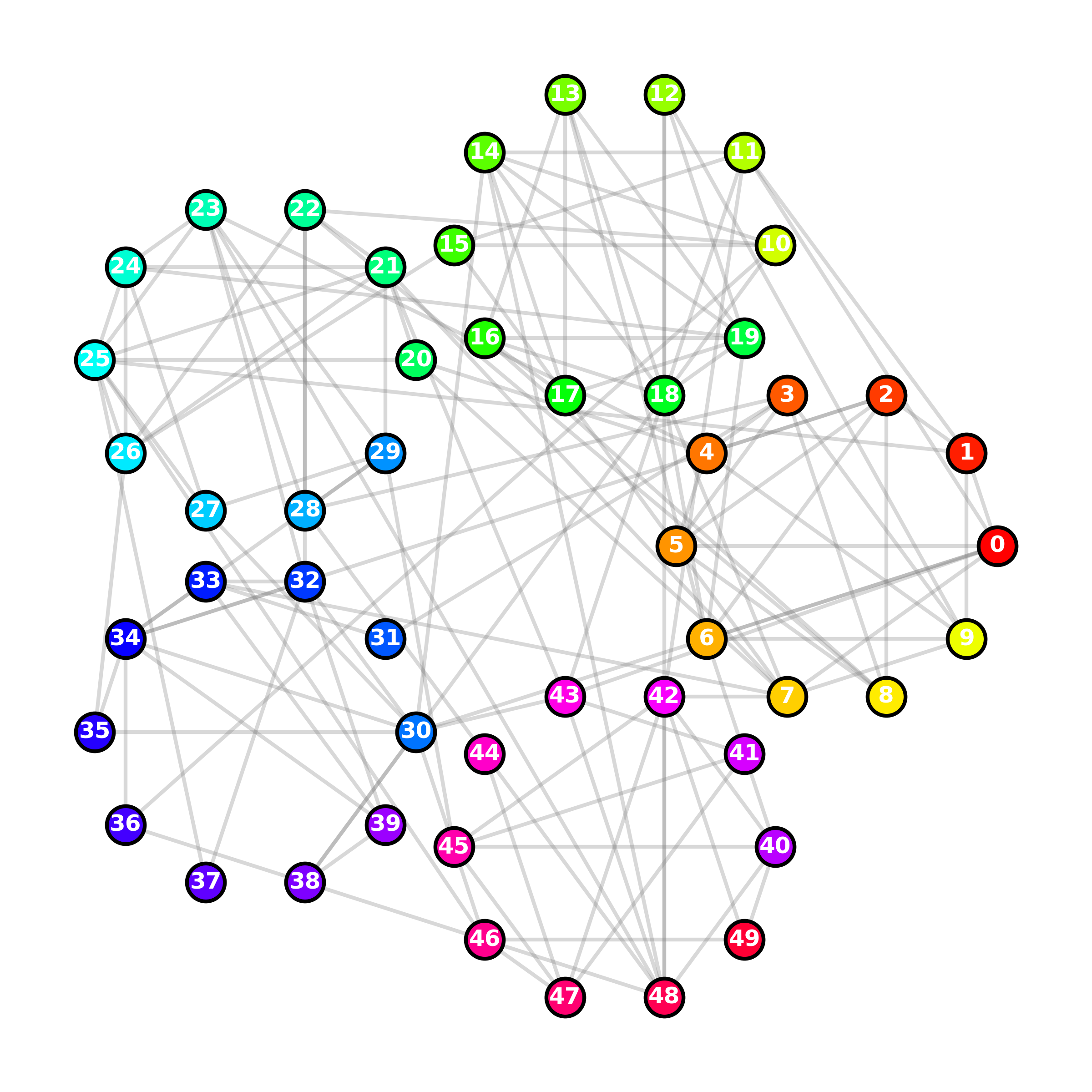}  \label{fig:5d} \\
        (c) Multi Consensus - Heterophily Pattern & (d) Individualized Consensus
    \end{tabular}
    }
    \caption{Visualisation of colour-coded node labels learned by GODNF for different convergence configurations.} 
    \label{fig:consensus_states}
\end{figure}

In the single consensus (Figure 5a), GODNF learns to assign identical labels across all nodes. In the multi-consensus scenario with homophily labels (Figure 5b), our model correctly identifies the pattern with neighbors having similar labels. In contrast, the multi-consensus configuration with heterophily labels (Figure 5c) yields more intricate patterns, with neighboring nodes exhibiting markedly different labels. Finally, the individualized consensus (Figure 5d) shows the highest color diversity, capturing the personalized nature of node consensus. 

\subsubsection{Additional Experiments} An experiment illustrating the convergence behaviour of GODNF variants, validating Theorem \ref{thm:convergence}, is included in the Appendix \ref{sec:visualization}.

%% file: Sections/conclusion.tex
\section{Conclusion and Future Work}

We proposed a novel neural diffusion framework based on opinion dynamics. Our method generalizes multiple opinion dynamic models, providing an efficient and expressive solution that advances existing GNNs. We provided a comprehensive theoretical analysis of our framework's convergence properties. The experiments on node classification and influence estimation tasks demonstrated the effectiveness and efficiency of our solution. 

In our existing framework, the current feature retention parameter $\alpha$ is designed as a global hyperparameter. This design choice offers parameter efficiency and ensures stable convergence properties. In our future work, we plan to investigate learning this parameter in a node-specific manner from the underlying data distribution. Additionally, we aim to extend our evaluation to more challenging learning settings, such as transfer learning and out-of-distribution generalization.



%% file: Sections/appendix.tex
\section*{Appendix}

\setcounter{section}{0}
\renewcommand{\thesection}{\Alph{section}}

\section{Proofs}

In this section, we present the proofs for the theorems discussed in the main content. 

\convergence*

\begin{proof}
We work in the Banach space \( (\mathbb{R}^{n \times d}, \| \cdot \|) \). \( W(t) \to W^* \), which implies $M(t) \to M^*$, where $M^*$ is a fixed matrix. Since \( \|M(t)\|_{\mathrm{op}} < 1 \) for all \( t \), we have \( \|M^*\|_{\mathrm{op}} < 1 \). Let \( X^* \) satisfy:
\[
X^* = \alpha X^* + (1-\alpha) \Lambda X(0) + (1-\alpha) M^* X^*
\]

Define \( A(t) = \alpha I + (1-\alpha)M(t) \) and \( A^* = \alpha I + (1-\alpha)M^* \).

For any \( t \), we have:
\[
\|A(t)\|_{\mathrm{op}} \leq \alpha + (1-\alpha)\|M(t)\|_{\mathrm{op}} < \alpha + (1-\alpha) = 1
\]

Let \( \beta = \sup_t \|A(t)\|_{\mathrm{op}} < 1 \). For the error \( E(t) = X(t) - X^* \), we have:
\[
E(t+1) = A(t)X(t) - A^*X^* + (1-\alpha)\Lambda X(0) - (1-\alpha)\Lambda X(0)
\]
\[
E(t+1) = A(t)E(t) + (A(t) - A^*)X^*.
\]

Taking norms:
\[
\|E(t+1)\| \leq \|A(t)\|_{\mathrm{op}}\|E(t)\| + \|A(t) - A^*\|_{\mathrm{op}}\|X^*\|
\]
\[
\|E(t+1)\| \leq \beta\|E(t)\| + (1-\alpha)\|M(t) - M^*\|_{\mathrm{op}}\|X^*\|
\]

Since \( M(t) \to M^* \), for any \( \tau > 0 \), there exists \( T \) such that for all \( t \geq T \), \( \|M(t) - M^*\|_{\mathrm{op}} < \frac{\tau}{(1-\alpha)\|X^*\|} \) (assuming \( X^* \neq 0 \)).

For \( t \geq T \):
\[
\|E(t+1)\| \leq \beta\|E(t)\| + \tau
\]
This is a contraction mapping with a small perturbation, as \( t \to \infty \) and \( \tau \to 0 \). Thus, according to the Banach fixed point theorem \cite{mannan2021study}, \( X(t) \) converges to \( X^* \).
\end{proof}

\singleconsensus*

\begin{proof}
 With these parameter configurations, the update rule of GODNF becomes \( X(t+1) = W^* X(t) \). Given that \( W^* \) is a row-stochastic matrix, this would be equivalent to the FD model. As proven in DeGroot et al. \cite{degroot1974reaching}, FD model converges to single consensus. Therefore, GODNF will exhibit the same behavior.

\end{proof}

\multiconsensus*

\begin{proof}

The fixed point of GODNF can be written as,

$X^* = (I - M^*)^{-1} \Lambda X(0)$. Since   $||M(t)||_{\mathrm{op}} < 1$ for all $t$, $||M^*||_{\mathrm{op}} < 1$
and $(I-M^*)$ invertible. Therefore, the existence of the above fixed point is guaranteed. Let $B^* = (I-M^*)^{-1}$,  and $X^* = B^* \Lambda X(0)$.  Given $M^*$ has block structure,  we can write $B^*$ as follows:

\begin{equation*}
B^* = \begin{bmatrix} 
B_1 & 0 & \cdots & 0 \\
0 & B_2 & \cdots & 0 \\
\vdots & \vdots & \ddots & \vdots \\
0 & 0 & \cdots & B_k
\end{bmatrix}
\end{equation*}

This leads to set of independent components $X_1^* = B_1^*  \Lambda_1 X_1(0)$, $X_2^* = B_2^*\Lambda_2 X_2(0)$, $\dots$ $X_k^* = B_k^*\Lambda_k X_k(0)$.  Each component reaches its own consensus independently, leading to multiple consensuses when initial conditions and parameters differ for at least two components.

\end{proof}

\individualizedconsensus*

\begin{proof}

The fixed point of GODNF is $X^* = (I - M^*)^{-1} \Lambda X(0)$ where $M^* =  (1 - \Lambda) (W^* - \mu L_g)$. Since   $||M(t)||_{\mathrm{op}} < 1$ for all $t$, $||M^*||_{\mathrm{op}} < 1$  and $(I-M^*)$ invertible. This ensures the existence of the above fixed point. When $\lambda_i$ approaches 1 for all nodes, $M \approx 0$, giving:
\[
X^* \approx (I-0)^{-1}\Lambda X(0) = \Lambda X(0)
\]

Since $\Lambda$ has positive diagonal entries close to 1, each node's convergence value closely preserves its initial feature. With $X(0)$ containing distinct initial features, this naturally leads to individualized consensus. 

\end{proof}

\section{Mapping of GODNF Components to Opinion Dynamic Models} 
\label{sec:mapping}

\begin{table}[ht]
\scalebox{0.8}{
\centering
\begin{tabular}{l|c|c|c|c}
\hline
\multirow{2}{*}{Model} & Current & Initial & Dynamic & Structural \\
 & Feature & Feature & Neighborhood & Regularization \\
 & Retention & Attachment & Influence & \\

\hline
FD  & $\times$ & $\times$ & $\times$ & $\times$ \\
\hline
FJ & $\times$ & $\checkmark$ & $\times$ & $\times$ \\
\hline
HK & $\checkmark$ & $\times$ & $\checkmark$ & $\times$ \\
\hline
GODNF$_{Static}$ & $\checkmark$ & $\checkmark$ & $\times$ & $\checkmark$ \\
\hline
GODNF$_{Dynamic}$ & $\checkmark$ & $\checkmark$ & $\checkmark$ & $\checkmark$ \\
\hline
\end{tabular}}
\caption{Comparison of Opinion Dynamics Models and GODNF Components. Note that for FD and FJ models, we consider formulations without self-loops, hence no current feature retention. While these models incorporate neighbourhood influence, they utilise static (time-invariant) connection weights and topologies, thus marked as lacking \emph{Dynamic Neighbourhood Influence}. Further, we consider the formulation where HK includes agent's current opinion in bounded confidence averaging. }
\label{tab:comparison}
\end{table}

In Table \ref{tab:comparison}, we demonstrate the versatility of GODNF by showing how it generalizes various opinion dynamics models, aligning its components with the mechanisms of these models.

Our neural framework offers a unified approach to opinion dynamics. It encapsulates the diverse and nuanced diffusion dynamics inherent in these distinct models, enhancing expressiveness.

\section{Supplementary Experimental Details}
\label{sec:experimental_details}

\subsection{Model Hyperparameters} In our experiments, we employ a Grid search to find optimal hyperparameters of GODNF in the following ranges: number of layers $\in \{2, 3, 4\}$, learning rate $\in \{ 5e-3, 1e-2\}$, weight decay $\in \{5e-4, 1e-1\}$, dropout $\in \{0.4, 0.5\}$, and the hidden dimension size $\in \{32, 64, 256\}$. GODNF model hyperparameter, $\alpha$ is chosen from \{0.1, 0.3, 0.6, 0.9\}. Note that $\Lambda$, and $\mu$ parameters in our model are learned from the ground truth. In node classification, we train models for up to 1000 epochs, while for influence estimation, we use 200 epochs. We utilize the Adam algorithm \cite{kingma2014adam} as the optimizer.

\subsection{Implementation Details}

All experiments were conducted on a Linux server with an Intel Xeon W-2175 2.50 GHz processor, comprising 28 cores, an NVIDIA RTX A6000 GPU, and 512 GB of RAM. We use the following Python libraries for our implementation: PyTorch version 2.3.1, torchvision 0.18.1, torchaudio 2.3.1, torch-geometric 2.7.0, torch-cluster 1.6.3, torch-scatter 2.0.9, and torch\_sparse 0.6.18.

\subsection{Benchmark Dataset Statistics}

Dataset statistics for node classification and node regression tasks are depicted in table \ref{tab:dataset_stats_classification} and table \ref{tab:dataset_stats_regression}, respectively.

\begin{table}[h]
\centering
\scalebox{0.8}{
\begin{tabular}{lllccc}
\hline
\textbf{Type} & \textbf{Dataset} & \textbf{Homophily} & \textbf{\# Nodes} & \textbf{\# Edges} & \textbf{\# Classes} \\
& & \textbf{Level} & & & \\
\hline
Heterophily & Texas & 0.06 & 183 & 309 & 5 \\
Heterophily & Cornell & 0.12 & 183 & 295 & 5 \\
Heterophily & Wisconsin & 0.18 & 251 & 499 & 5 \\
Heterophily & Film & 0.22 & 7,600 & 33,544 & 5 \\
Heterophily & Amazon-rating & 0.38 & 24,492 & 186,100 & 5 \\
\hline
Homophily & Cora Full & 0.57 & 19,793 & 126,842 & 70 \\
Homophily & Citeseer & 0.74 & 3,312 & 4,732 & 6 \\
Homophily & Cora-ML & 0.79 & 2,995 & 16,316 & 7 \\
Homophily & PubMed & 0.80 & 19,717 & 44,338 & 3 \\
Homophily & DBLP & 0.83 & 17,716 & 105,734 & 4 \\
\hline
\end{tabular}
}
\caption{Dataset statistics for node classification}
\label{tab:dataset_stats_classification}
\end{table}

\begin{table}[h]
\centering
\scalebox{0.9}{
\begin{tabular}{lcc}
\hline
\textbf{Dataset} & \textbf{\# Nodes} & \textbf{\# Edges} \\
\hline
Jazz & 198 & 2,742 \\
Cora-ML & 2,810 & 7,981 \\
Network Science & 1,565 & 13,532 \\
Power Grid & 4,941 & 6,594 \\
\hline
\end{tabular}
}
\caption{Dataset statistics for influence estimation task}
\label{tab:dataset_stats_regression}
\end{table}

\subsection{Additional Experiments}

 \subsubsection{Convergence and Visualization Analysis} 
 \label{sec:visualization}

 We conduct experiments to empirically validate the convergence properties of GODNF, performing node classification tasks on the Cora Full dataset. Figure \ref{fig:embedding_change} shows the convergence behaviour of GODNF variants through step-wise average embedding changes over evolution time. The exponential decay pattern demonstrates that both GODNF variants achieve stable convergence, with average embedding changes decreasing rapidly from initial large adjustments to minimal fluctuations by step 7-8, confirming the theoretical convergence guarantees presented in Theorem \ref{thm:convergence}.

 \begin{figure}[h]
    \centering
    \begin{subfigure}{0.48\linewidth}
        \centering
        \includegraphics[width=\linewidth]{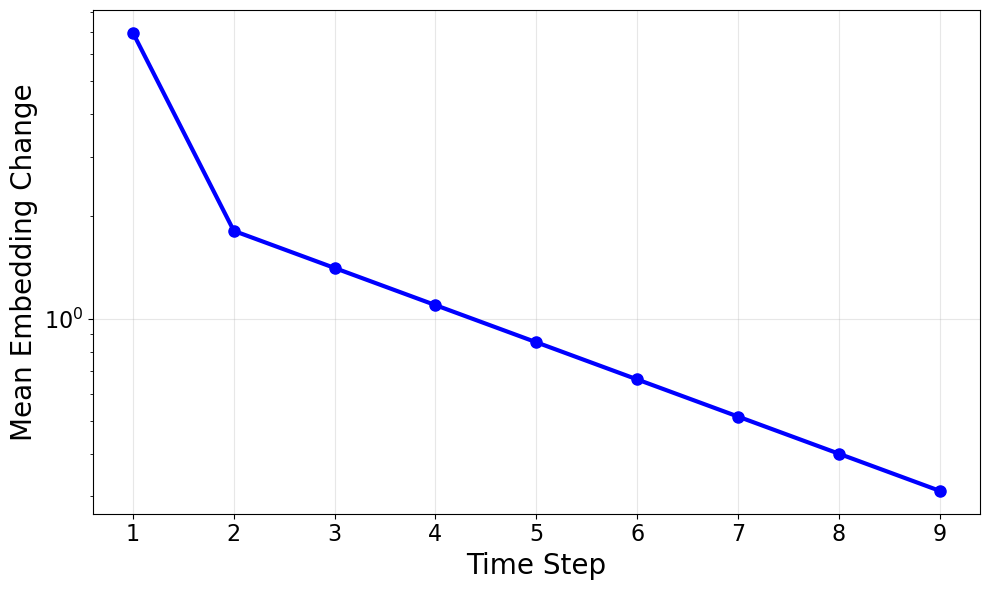}
        \caption{GODNF$_\text{Static}$}
    \end{subfigure}
    \hfill
    \begin{subfigure}{0.48\linewidth}
        \centering
        \includegraphics[width=\linewidth]{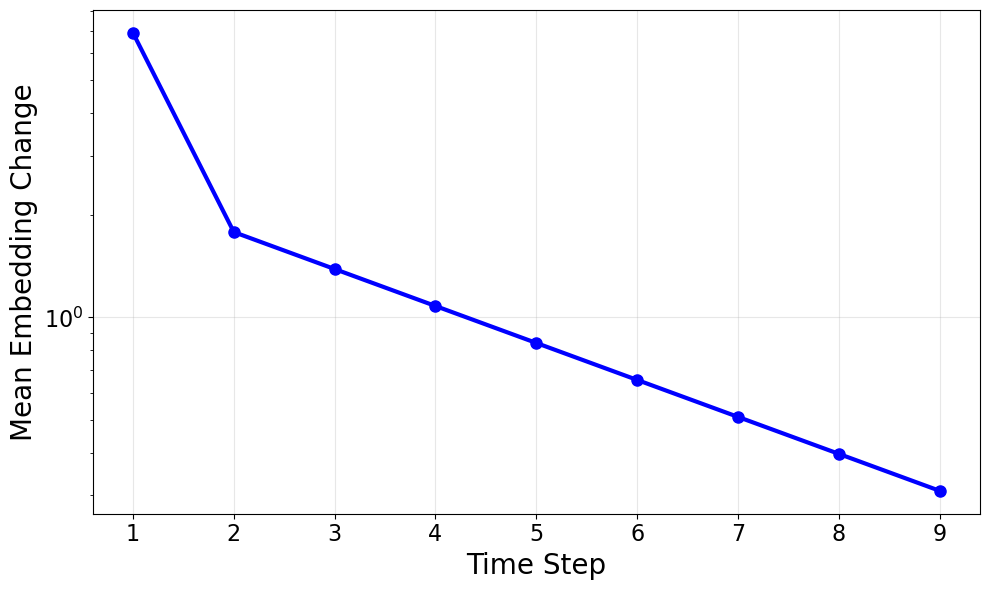}
       \caption{GODNF$_\text{Dynamic}$}
    \end{subfigure}
    \caption{Step-wise embedding changes of GODNF over time steps for the Cora Full for node classification.}
    \label{fig:embedding_change}
\end{figure}

 \begin{figure*}
    \centering
    \scalebox{0.92}{ 
        \begin{minipage}{\linewidth}
            \centering
            \begin{subfigure}[b]{\linewidth}
                \centering
                \includegraphics[width=\linewidth]{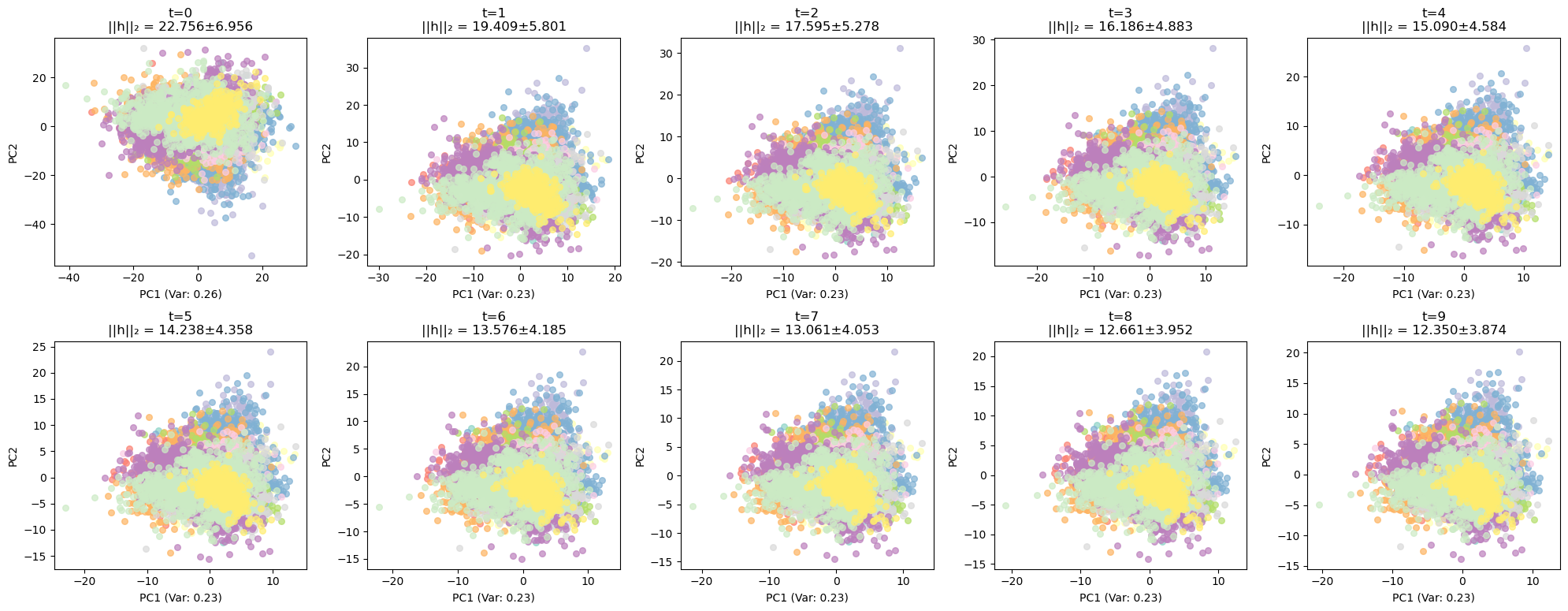}
                \caption{GODNF$_\text{Static}$}
            \end{subfigure}
            \vspace{0.5cm}
            \begin{subfigure}[b]{\linewidth}
                \centering
                \includegraphics[width=\linewidth]{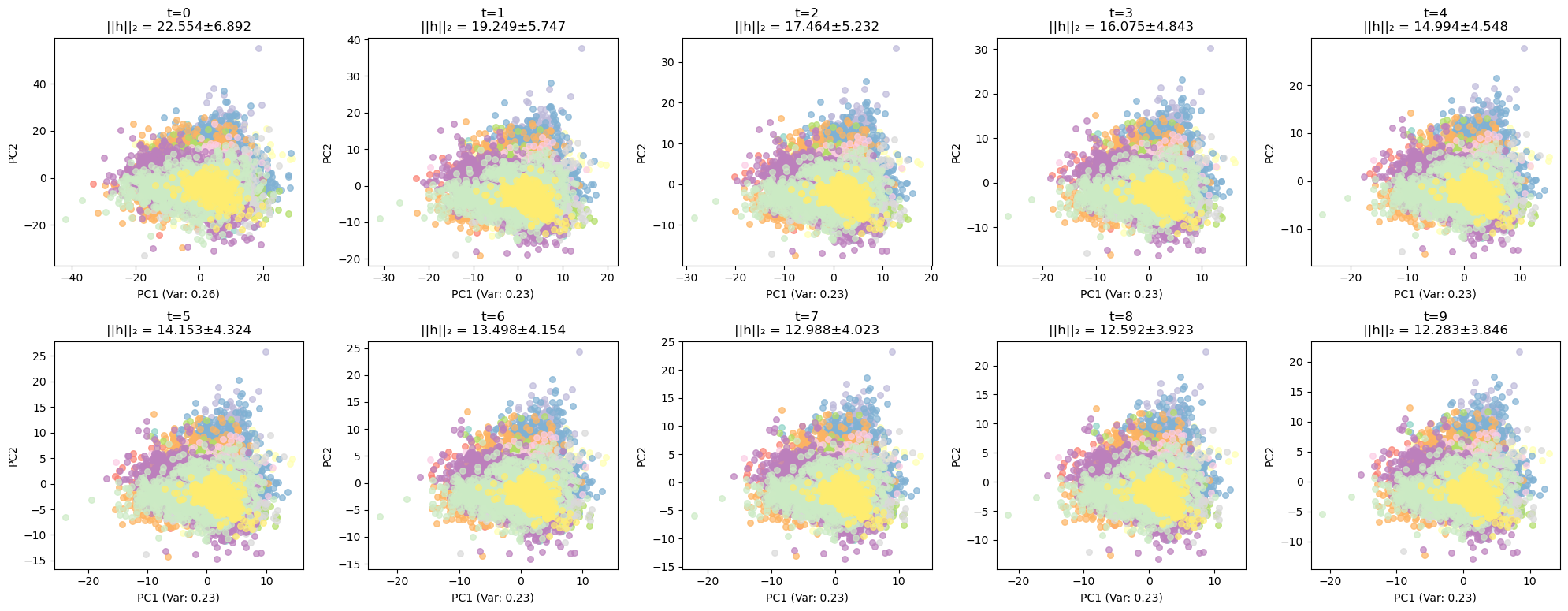}
                \caption{GODNF$_\text{Dynamic}$}
            \end{subfigure}
        \end{minipage}
    }
    \caption{Embedding evolution for Cora Full dataset over multiple time steps. $\|h\|_2$ represents the average L2 norm across all individual node embeddings, indicating the mean embedding magnitude. PCA (Var) shows the percentage of original high-dimensional embedding variance captured in the 2D visualization.}
    \label{fig:embedding_convergence}
\end{figure*}

   Next, we provide a visualisation of embedding convergence. Figure \ref{fig:embedding_convergence} shows the embedding evolution of GODNF variants in 2D principal component analysis (PCA) space across time steps. The embedding evolution shows that GODNF successfully learns better node representations over time. Initially, nodes are scattered randomly in the embedding space, but as iterations progress, nodes of the same label gradually move closer together to form distinct clusters. The decreasing embedding norm values (i.e., $||h||_2$) indicate that embeddings become more stable. At the same time, the improved visual separation between different colored clusters demonstrates that the model learns to distinguish between different node classes effectively. The static variant of GODNF exhibits a slightly smoother progression of embeddings and more consistent cluster boundaries. In contrast, the dynamic variant adjusts weights for each time step, resulting in slightly higher initial variability, but it achieves comparable final convergence quality.

   \subsubsection{Robustness to adversarial attacks} \label{sec:adversarial_attack} Figure \ref{fig:attack_benchmark_film} compares the node classification performance of GODNF against the baselines of the Film dataset under various adversarial attacks. This is in addition to the Citeseer dataset provided in the main content. Consistent with the Citeseer results, GODNF variants also outperform all baselines on the Film dataset under various adversarial attack scenarios.

\begin{figure*}
    \centering
    \includegraphics[width=\linewidth]{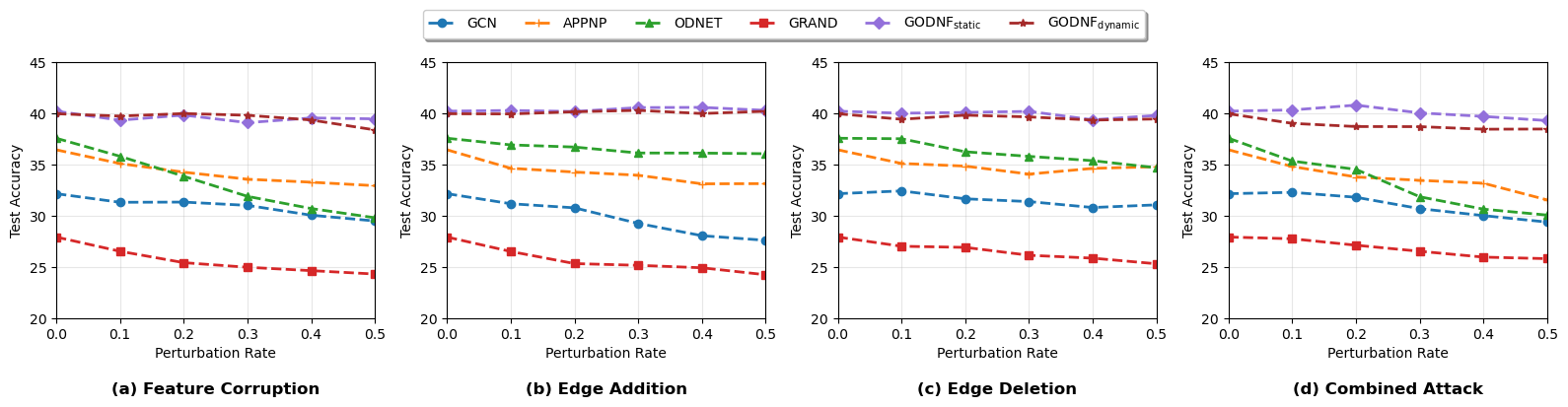}
    \caption{Node classification performance for Film dataset under different adversarial attacks.}
    \label{fig:attack_benchmark_film}
\end{figure*}

%% file: wsdm25.bbl

\begin{thebibliography}{72}


\ifx \showCODEN    \undefined \def \showCODEN     #1{\unskip}     \fi
\ifx \showISBNx    \undefined \def \showISBNx     #1{\unskip}     \fi
\ifx \showISBNxiii \undefined \def \showISBNxiii  #1{\unskip}     \fi
\ifx \showISSN     \undefined \def \showISSN      #1{\unskip}     \fi
\ifx \showLCCN     \undefined \def \showLCCN      #1{\unskip}     \fi
\ifx \shownote     \undefined \def \shownote      #1{#1}          \fi
\ifx \showarticletitle \undefined \def \showarticletitle #1{#1}   \fi
\ifx \showURL      \undefined \def \showURL       {\relax}        \fi
\providecommand\bibfield[2]{#2}
\providecommand\bibinfo[2]{#2}
\providecommand\natexlab[1]{#1}
\providecommand\showeprint[2][]{arXiv:#2}

\bibitem[Abbe(2018)]%
        {abbe2018community}
\bibfield{author}{\bibinfo{person}{Emmanuel Abbe}.} \bibinfo{year}{2018}\natexlab{}.
\newblock \showarticletitle{Community detection and stochastic block models: recent developments}.
\newblock \bibinfo{journal}{\emph{Journal of Machine Learning Research}} \bibinfo{volume}{18}, \bibinfo{number}{177} (\bibinfo{year}{2018}), \bibinfo{pages}{1--86}.
\newblock


\bibitem[Biondi et~al\mbox{.}(2023)]%
        {biondi2023dynamics}
\bibfield{author}{\bibinfo{person}{Elisabetta Biondi}, \bibinfo{person}{Chiara Boldrini}, \bibinfo{person}{Andrea Passarella}, {and} \bibinfo{person}{Marco Conti}.} \bibinfo{year}{2023}\natexlab{}.
\newblock \showarticletitle{Dynamics of opinion polarization}.
\newblock \bibinfo{journal}{\emph{IEEE Transactions on Systems, Man, and Cybernetics: Systems}} \bibinfo{volume}{53}, \bibinfo{number}{9} (\bibinfo{year}{2023}), \bibinfo{pages}{5381--5392}.
\newblock


\bibitem[Bodnar et~al\mbox{.}(2022)]%
        {bodnar2022neural}
\bibfield{author}{\bibinfo{person}{Cristian Bodnar}, \bibinfo{person}{Francesco Di~Giovanni}, \bibinfo{person}{Benjamin Chamberlain}, \bibinfo{person}{Pietro Lio}, {and} \bibinfo{person}{Michael Bronstein}.} \bibinfo{year}{2022}\natexlab{}.
\newblock \showarticletitle{Neural sheaf diffusion: A topological perspective on heterophily and oversmoothing in gnns}.
\newblock \bibinfo{journal}{\emph{Advances in Neural Information Processing Systems}}  \bibinfo{volume}{35} (\bibinfo{year}{2022}), \bibinfo{pages}{18527--18541}.
\newblock


\bibitem[Bodnar et~al\mbox{.}(2021)]%
        {bodnar2021weisfeiler}
\bibfield{author}{\bibinfo{person}{Cristian Bodnar}, \bibinfo{person}{Fabrizio Frasca}, \bibinfo{person}{Nina Otter}, \bibinfo{person}{Yuguang Wang}, \bibinfo{person}{Pietro Lio}, \bibinfo{person}{Guido~F Montufar}, {and} \bibinfo{person}{Michael Bronstein}.} \bibinfo{year}{2021}\natexlab{}.
\newblock \showarticletitle{Weisfeiler and lehman go cellular: Cw networks}.
\newblock \bibinfo{journal}{\emph{Advances in neural information processing systems}}  \bibinfo{volume}{34} (\bibinfo{year}{2021}), \bibinfo{pages}{2625--2640}.
\newblock


\bibitem[Bojchevski and G{\"u}nnemann(2018)]%
        {bojchevski2018deep}
\bibfield{author}{\bibinfo{person}{Aleksandar Bojchevski} {and} \bibinfo{person}{Stephan G{\"u}nnemann}.} \bibinfo{year}{2018}\natexlab{}.
\newblock \showarticletitle{Deep Gaussian Embedding of Graphs: Unsupervised Inductive Learning via Ranking}. In \bibinfo{booktitle}{\emph{International Conference on Learning Representations}}.
\newblock


\bibitem[Cai et~al\mbox{.}(2023)]%
        {cai2023expressive}
\bibfield{author}{\bibinfo{person}{Xuheng Cai}, \bibinfo{person}{Lianghao Xia}, \bibinfo{person}{Xubin Ren}, {and} \bibinfo{person}{Chao Huang}.} \bibinfo{year}{2023}\natexlab{}.
\newblock \showarticletitle{How expressive are graph neural networks in recommendation?}. In \bibinfo{booktitle}{\emph{Proceedings of the 32nd ACM International Conference on Information and Knowledge Management}}. \bibinfo{pages}{173--182}.
\newblock


\bibitem[Chamberlain et~al\mbox{.}(2021a)]%
        {chamberlain2021beltrami}
\bibfield{author}{\bibinfo{person}{Benjamin Chamberlain}, \bibinfo{person}{James Rowbottom}, \bibinfo{person}{Davide Eynard}, \bibinfo{person}{Francesco Di~Giovanni}, \bibinfo{person}{Xiaowen Dong}, {and} \bibinfo{person}{Michael Bronstein}.} \bibinfo{year}{2021}\natexlab{a}.
\newblock \showarticletitle{Beltrami flow and neural diffusion on graphs}.
\newblock \bibinfo{journal}{\emph{Advances in Neural Information Processing Systems}}  \bibinfo{volume}{34} (\bibinfo{year}{2021}), \bibinfo{pages}{1594--1609}.
\newblock


\bibitem[Chamberlain et~al\mbox{.}(2021b)]%
        {chamberlain2021grand}
\bibfield{author}{\bibinfo{person}{Ben Chamberlain}, \bibinfo{person}{James Rowbottom}, \bibinfo{person}{Maria~I Gorinova}, \bibinfo{person}{Michael Bronstein}, \bibinfo{person}{Stefan Webb}, {and} \bibinfo{person}{Emanuele Rossi}.} \bibinfo{year}{2021}\natexlab{b}.
\newblock \showarticletitle{Grand: Graph neural diffusion}. In \bibinfo{booktitle}{\emph{International conference on machine learning}}. PMLR, \bibinfo{pages}{1407--1418}.
\newblock


\bibitem[Chen et~al\mbox{.}(2025)]%
        {chen2025graph}
\bibfield{author}{\bibinfo{person}{Jialong Chen}, \bibinfo{person}{Bowen Deng}, \bibinfo{person}{Chuan Chen}, \bibinfo{person}{Zibin Zheng}, {et~al\mbox{.}}} \bibinfo{year}{2025}\natexlab{}.
\newblock \showarticletitle{Graph Neural Ricci Flow: Evolving Feature from a Curvature Perspective}. In \bibinfo{booktitle}{\emph{The Thirteenth International Conference on Learning Representations}}.
\newblock


\bibitem[Chen et~al\mbox{.}(2017)]%
        {chen2017opinion}
\bibfield{author}{\bibinfo{person}{Xi Chen}, \bibinfo{person}{Xiao Zhang}, \bibinfo{person}{Yong Xie}, {and} \bibinfo{person}{Wei Li}.} \bibinfo{year}{2017}\natexlab{}.
\newblock \showarticletitle{Opinion Dynamics of Social-Similarity-Based Hegselmann--Krause Model}.
\newblock \bibinfo{journal}{\emph{Complexity}} \bibinfo{volume}{2017}, \bibinfo{number}{1} (\bibinfo{year}{2017}), \bibinfo{pages}{1820257}.
\newblock


\bibitem[Chien et~al\mbox{.}(2021)]%
        {chienadaptive}
\bibfield{author}{\bibinfo{person}{Eli Chien}, \bibinfo{person}{Jianhao Peng}, \bibinfo{person}{Pan Li}, {and} \bibinfo{person}{Olgica Milenkovic}.} \bibinfo{year}{2021}\natexlab{}.
\newblock \showarticletitle{Adaptive Universal Generalized PageRank Graph Neural Network}. In \bibinfo{booktitle}{\emph{International Conference on Learning Representations}}.
\newblock


\bibitem[Choi et~al\mbox{.}(2023)]%
        {choi2023gread}
\bibfield{author}{\bibinfo{person}{Jeongwhan Choi}, \bibinfo{person}{Seoyoung Hong}, \bibinfo{person}{Noseong Park}, {and} \bibinfo{person}{Sung-Bae Cho}.} \bibinfo{year}{2023}\natexlab{}.
\newblock \showarticletitle{Gread: Graph neural reaction-diffusion networks}. In \bibinfo{booktitle}{\emph{International Conference on Machine Learning}}. PMLR, \bibinfo{pages}{5722--5747}.
\newblock


\bibitem[Das et~al\mbox{.}(2014)]%
        {das2014modeling}
\bibfield{author}{\bibinfo{person}{Abhimanyu Das}, \bibinfo{person}{Sreenivas Gollapudi}, {and} \bibinfo{person}{Kamesh Munagala}.} \bibinfo{year}{2014}\natexlab{}.
\newblock \showarticletitle{Modeling opinion dynamics in social networks}. In \bibinfo{booktitle}{\emph{Proceedings of the 7th ACM international conference on Web search and data mining}}. \bibinfo{pages}{403--412}.
\newblock


\bibitem[DeGroot(1974)]%
        {degroot1974reaching}
\bibfield{author}{\bibinfo{person}{Morris~H DeGroot}.} \bibinfo{year}{1974}\natexlab{}.
\newblock \showarticletitle{Reaching a consensus}.
\newblock \bibinfo{journal}{\emph{Journal of the American Statistical association}} \bibinfo{volume}{69}, \bibinfo{number}{345} (\bibinfo{year}{1974}), \bibinfo{pages}{118--121}.
\newblock


\bibitem[Eliasof et~al\mbox{.}(2021)]%
        {eliasof2021pde}
\bibfield{author}{\bibinfo{person}{Moshe Eliasof}, \bibinfo{person}{Eldad Haber}, {and} \bibinfo{person}{Eran Treister}.} \bibinfo{year}{2021}\natexlab{}.
\newblock \showarticletitle{Pde-gcn: Novel architectures for graph neural networks motivated by partial differential equations}.
\newblock \bibinfo{journal}{\emph{Advances in neural information processing systems}}  \bibinfo{volume}{34} (\bibinfo{year}{2021}), \bibinfo{pages}{3836--3849}.
\newblock


\bibitem[Eliasof et~al\mbox{.}(2024)]%
        {eliasof2024feature}
\bibfield{author}{\bibinfo{person}{Moshe Eliasof}, \bibinfo{person}{Eldad Haber}, {and} \bibinfo{person}{Eran Treister}.} \bibinfo{year}{2024}\natexlab{}.
\newblock \showarticletitle{Feature transportation improves graph neural networks}. In \bibinfo{booktitle}{\emph{Proceedings of the AAAI Conference on Artificial Intelligence}}, Vol.~\bibinfo{volume}{38}. \bibinfo{pages}{11874--11882}.
\newblock


\bibitem[Friedkin and Johnsen(1990)]%
        {friedkin1990social}
\bibfield{author}{\bibinfo{person}{Noah~E Friedkin} {and} \bibinfo{person}{Eugene~C Johnsen}.} \bibinfo{year}{1990}\natexlab{}.
\newblock \showarticletitle{Social influence and opinions}.
\newblock \bibinfo{journal}{\emph{Journal of mathematical sociology}} \bibinfo{volume}{15}, \bibinfo{number}{3-4} (\bibinfo{year}{1990}), \bibinfo{pages}{193--206}.
\newblock


\bibitem[Gao et~al\mbox{.}(2022)]%
        {gao2022graph}
\bibfield{author}{\bibinfo{person}{Chen Gao}, \bibinfo{person}{Xiang Wang}, \bibinfo{person}{Xiangnan He}, {and} \bibinfo{person}{Yong Li}.} \bibinfo{year}{2022}\natexlab{}.
\newblock \showarticletitle{Graph neural networks for recommender system}. In \bibinfo{booktitle}{\emph{Proceedings of the fifteenth ACM international conference on web search and data mining}}. \bibinfo{pages}{1623--1625}.
\newblock


\bibitem[Gasteiger et~al\mbox{.}(2018)]%
        {gasteiger2018predict}
\bibfield{author}{\bibinfo{person}{Johannes Gasteiger}, \bibinfo{person}{Aleksandar Bojchevski}, {and} \bibinfo{person}{Stephan G{\"u}nnemann}.} \bibinfo{year}{2018}\natexlab{}.
\newblock \showarticletitle{Predict then Propagate: Graph Neural Networks meet Personalized PageRank}. In \bibinfo{booktitle}{\emph{International Conference on Learning Representations}}.
\newblock


\bibitem[Gasteiger et~al\mbox{.}(2019)]%
        {gasteiger2019diffusion}
\bibfield{author}{\bibinfo{person}{Johannes Gasteiger}, \bibinfo{person}{Stefan Wei{\ss}enberger}, {and} \bibinfo{person}{Stephan G{\"u}nnemann}.} \bibinfo{year}{2019}\natexlab{}.
\newblock \showarticletitle{Diffusion improves graph learning}.
\newblock \bibinfo{journal}{\emph{Advances in neural information processing systems}}  \bibinfo{volume}{32} (\bibinfo{year}{2019}).
\newblock


\bibitem[Goldenberg et~al\mbox{.}(2001)]%
        {goldenberg2001talk}
\bibfield{author}{\bibinfo{person}{Jacob Goldenberg}, \bibinfo{person}{Barak Libai}, {and} \bibinfo{person}{Eitan Muller}.} \bibinfo{year}{2001}\natexlab{}.
\newblock \showarticletitle{Talk of the network: A complex systems look at the underlying process of word-of-mouth}.
\newblock \bibinfo{journal}{\emph{Marketing letters}}  \bibinfo{volume}{12} (\bibinfo{year}{2001}), \bibinfo{pages}{211--223}.
\newblock


\bibitem[Hamilton et~al\mbox{.}(2017)]%
        {hamilton2017inductive}
\bibfield{author}{\bibinfo{person}{Will Hamilton}, \bibinfo{person}{Zhitao Ying}, {and} \bibinfo{person}{Jure Leskovec}.} \bibinfo{year}{2017}\natexlab{}.
\newblock \showarticletitle{Inductive representation learning on large graphs}.
\newblock \bibinfo{journal}{\emph{Advances in neural information processing systems}}  \bibinfo{volume}{30} (\bibinfo{year}{2017}).
\newblock


\bibitem[Han et~al\mbox{.}(2023)]%
        {hancontinuous}
\bibfield{author}{\bibinfo{person}{Andi Han}, \bibinfo{person}{Dai Shi}, \bibinfo{person}{Lequan Lin}, {and} \bibinfo{person}{Junbin Gao}.} \bibinfo{year}{2023}\natexlab{}.
\newblock \showarticletitle{From Continuous Dynamics to Graph Neural Networks: Neural Diffusion and Beyond}.
\newblock \bibinfo{journal}{\emph{Transactions on Machine Learning Research}} (\bibinfo{year}{2023}).
\newblock


\bibitem[Hassani et~al\mbox{.}(2022)]%
        {hassani2022classical}
\bibfield{author}{\bibinfo{person}{Hossein Hassani}, \bibinfo{person}{Roozbeh Razavi-Far}, \bibinfo{person}{Mehrdad Saif}, \bibinfo{person}{Francisco Chiclana}, \bibinfo{person}{Ondrej Krejcar}, {and} \bibinfo{person}{Enrique Herrera-Viedma}.} \bibinfo{year}{2022}\natexlab{}.
\newblock \showarticletitle{Classical dynamic consensus and opinion dynamics models: A survey of recent trends and methodologies}.
\newblock \bibinfo{journal}{\emph{Information Fusion}}  \bibinfo{volume}{88} (\bibinfo{year}{2022}), \bibinfo{pages}{22--40}.
\newblock


\bibitem[Hevapathige et~al\mbox{.}(2025)]%
        {hevapathige2025depth}
\bibfield{author}{\bibinfo{person}{Asela Hevapathige}, \bibinfo{person}{Ahad~N Zehmakan}, {and} \bibinfo{person}{Qing Wang}.} \bibinfo{year}{2025}\natexlab{}.
\newblock \showarticletitle{Depth-Adaptive Graph Neural Networks via Learnable Bakry-{\'E}mery Curvature}. In \bibinfo{booktitle}{\emph{Proceedings of the 31st ACM SIGKDD Conference on Knowledge Discovery and Data Mining V. 2}}. \bibinfo{pages}{944--955}.
\newblock


\bibitem[Horn and Johnson(2012)]%
        {horn2012matrix}
\bibfield{author}{\bibinfo{person}{Roger~A Horn} {and} \bibinfo{person}{Charles~R Johnson}.} \bibinfo{year}{2012}\natexlab{}.
\newblock \bibinfo{booktitle}{\emph{Matrix analysis}}.
\newblock \bibinfo{publisher}{Cambridge university press}.
\newblock


\bibitem[Hu et~al\mbox{.}(2020)]%
        {hu2020open}
\bibfield{author}{\bibinfo{person}{Weihua Hu}, \bibinfo{person}{Matthias Fey}, \bibinfo{person}{Marinka Zitnik}, \bibinfo{person}{Yuxiao Dong}, \bibinfo{person}{Hongyu Ren}, \bibinfo{person}{Bowen Liu}, \bibinfo{person}{Michele Catasta}, {and} \bibinfo{person}{Jure Leskovec}.} \bibinfo{year}{2020}\natexlab{}.
\newblock \showarticletitle{Open graph benchmark: Datasets for machine learning on graphs}.
\newblock \bibinfo{journal}{\emph{Advances in neural information processing systems}}  \bibinfo{volume}{33} (\bibinfo{year}{2020}), \bibinfo{pages}{22118--22133}.
\newblock


\bibitem[Hyun et~al\mbox{.}(2024)]%
        {hyun2024lex}
\bibfield{author}{\bibinfo{person}{Woochang Hyun}, \bibinfo{person}{Insoo Lee}, {and} \bibinfo{person}{Bongwon Suh}.} \bibinfo{year}{2024}\natexlab{}.
\newblock \showarticletitle{LEX-GNN: Label-Exploring Graph Neural Network for Accurate Fraud Detection}. In \bibinfo{booktitle}{\emph{Proceedings of the 33rd ACM International Conference on Information and Knowledge Management}}. \bibinfo{pages}{3802--3806}.
\newblock


\bibitem[Jain et~al\mbox{.}(2023)]%
        {jain2023opinion}
\bibfield{author}{\bibinfo{person}{Lokesh Jain}, \bibinfo{person}{Rahul Katarya}, {and} \bibinfo{person}{Shelly Sachdeva}.} \bibinfo{year}{2023}\natexlab{}.
\newblock \showarticletitle{Opinion leaders for information diffusion using graph neural network in online social networks}.
\newblock \bibinfo{journal}{\emph{ACM Transactions on the Web}} \bibinfo{volume}{17}, \bibinfo{number}{2} (\bibinfo{year}{2023}), \bibinfo{pages}{1--37}.
\newblock


\bibitem[Kempe et~al\mbox{.}(2003)]%
        {kempe2003maximizing}
\bibfield{author}{\bibinfo{person}{David Kempe}, \bibinfo{person}{Jon Kleinberg}, {and} \bibinfo{person}{{\'E}va Tardos}.} \bibinfo{year}{2003}\natexlab{}.
\newblock \showarticletitle{Maximizing the spread of influence through a social network}. In \bibinfo{booktitle}{\emph{Proceedings of the ninth ACM SIGKDD international conference on Knowledge discovery and data mining}}. \bibinfo{pages}{137--146}.
\newblock


\bibitem[Kim et~al\mbox{.}(2022)]%
        {kim2022graph}
\bibfield{author}{\bibinfo{person}{Hwan Kim}, \bibinfo{person}{Byung~Suk Lee}, \bibinfo{person}{Won-Yong Shin}, {and} \bibinfo{person}{Sungsu Lim}.} \bibinfo{year}{2022}\natexlab{}.
\newblock \showarticletitle{Graph anomaly detection with graph neural networks: Current status and challenges}.
\newblock \bibinfo{journal}{\emph{IEEE Access}}  \bibinfo{volume}{10} (\bibinfo{year}{2022}), \bibinfo{pages}{111820--111829}.
\newblock


\bibitem[Kimura et~al\mbox{.}(2009)]%
        {kimura2009efficient}
\bibfield{author}{\bibinfo{person}{Masahiro Kimura}, \bibinfo{person}{Kazumi Saito}, {and} \bibinfo{person}{Hiroshi Motoda}.} \bibinfo{year}{2009}\natexlab{}.
\newblock \showarticletitle{Efficient Estimation of Influence Functions for SIS Model on Social Networks.}. In \bibinfo{booktitle}{\emph{IJCAI}}. \bibinfo{pages}{2046--2051}.
\newblock


\bibitem[Kingma(2015)]%
        {kingma2014adam}
\bibfield{author}{\bibinfo{person}{DP Kingma}.} \bibinfo{year}{2015}\natexlab{}.
\newblock \showarticletitle{Adam: a method for stochastic optimization}. In \bibinfo{booktitle}{\emph{International Conference on Learning Representations}}.
\newblock


\bibitem[Kipf and Welling(2017)]%
        {kipf2017semi}
\bibfield{author}{\bibinfo{person}{Thomas~N Kipf} {and} \bibinfo{person}{Max Welling}.} \bibinfo{year}{2017}\natexlab{}.
\newblock \showarticletitle{Semi-Supervised Classification with Graph Convolutional Networks}. In \bibinfo{booktitle}{\emph{International Conference on Learning Representations}}.
\newblock


\bibitem[Kumar et~al\mbox{.}(2022)]%
        {kumar2022influence}
\bibfield{author}{\bibinfo{person}{Sanjay Kumar}, \bibinfo{person}{Abhishek Mallik}, \bibinfo{person}{Anavi Khetarpal}, {and} \bibinfo{person}{Bhawani~Sankar Panda}.} \bibinfo{year}{2022}\natexlab{}.
\newblock \showarticletitle{Influence maximization in social networks using graph embedding and graph neural network}.
\newblock \bibinfo{journal}{\emph{Information Sciences}}  \bibinfo{volume}{607} (\bibinfo{year}{2022}), \bibinfo{pages}{1617--1636}.
\newblock


\bibitem[Lee and Jung(2023)]%
        {lee2023time}
\bibfield{author}{\bibinfo{person}{Jong-whi Lee} {and} \bibinfo{person}{Jinhong Jung}.} \bibinfo{year}{2023}\natexlab{}.
\newblock \showarticletitle{Time-aware random walk diffusion to improve dynamic graph learning}. In \bibinfo{booktitle}{\emph{Proceedings of the AAAI Conference on Artificial Intelligence}}, Vol.~\bibinfo{volume}{37}. \bibinfo{pages}{8473--8481}.
\newblock


\bibitem[Leonard et~al\mbox{.}(2024)]%
        {leonard2024fast}
\bibfield{author}{\bibinfo{person}{Naomi~Ehrich Leonard}, \bibinfo{person}{Anastasia Bizyaeva}, {and} \bibinfo{person}{Alessio Franci}.} \bibinfo{year}{2024}\natexlab{}.
\newblock \showarticletitle{Fast and flexible multiagent decision-making}.
\newblock \bibinfo{journal}{\emph{Annual Review of Control, Robotics, and Autonomous Systems}}  \bibinfo{volume}{7} (\bibinfo{year}{2024}).
\newblock


\bibitem[Li et~al\mbox{.}(2021)]%
        {li2021training}
\bibfield{author}{\bibinfo{person}{Guohao Li}, \bibinfo{person}{Matthias M{\"u}ller}, \bibinfo{person}{Bernard Ghanem}, {and} \bibinfo{person}{Vladlen Koltun}.} \bibinfo{year}{2021}\natexlab{}.
\newblock \showarticletitle{Training graph neural networks with 1000 layers}. In \bibinfo{booktitle}{\emph{International conference on machine learning}}. PMLR, \bibinfo{pages}{6437--6449}.
\newblock


\bibitem[Li et~al\mbox{.}(2025)]%
        {li2025unigo}
\bibfield{author}{\bibinfo{person}{Hao Li}, \bibinfo{person}{Hao Jiang}, \bibinfo{person}{Yuke Zheng}, \bibinfo{person}{Hao Sun}, {and} \bibinfo{person}{Wenying Gong}.} \bibinfo{year}{2025}\natexlab{}.
\newblock \showarticletitle{UniGO: A Unified Graph Neural Network for Modeling Opinion Dynamics on Graphs}. In \bibinfo{booktitle}{\emph{Proceedings of the ACM on Web Conference 2025}}. \bibinfo{pages}{530--540}.
\newblock


\bibitem[Li et~al\mbox{.}(2023b)]%
        {li2023influence}
\bibfield{author}{\bibinfo{person}{Hui Li}, \bibinfo{person}{Susu Yang}, \bibinfo{person}{Mengting Xu}, \bibinfo{person}{Sourav~S Bhowmick}, {and} \bibinfo{person}{Jiangtao Cui}.} \bibinfo{year}{2023}\natexlab{b}.
\newblock \showarticletitle{Influence maximization in social networks: A survey}.
\newblock \bibinfo{journal}{\emph{arXiv preprint arXiv:2309.04668}} (\bibinfo{year}{2023}).
\newblock


\bibitem[Li et~al\mbox{.}(2023c)]%
        {li2023lgm}
\bibfield{author}{\bibinfo{person}{Pengbo Li}, \bibinfo{person}{Hang Yu}, \bibinfo{person}{Xiangfeng Luo}, {and} \bibinfo{person}{Jia Wu}.} \bibinfo{year}{2023}\natexlab{c}.
\newblock \showarticletitle{LGM-GNN: A local and global aware memory-based graph neural network for fraud detection}.
\newblock \bibinfo{journal}{\emph{IEEE Transactions on Big Data}} \bibinfo{volume}{9}, \bibinfo{number}{4} (\bibinfo{year}{2023}), \bibinfo{pages}{1116--1127}.
\newblock


\bibitem[Li et~al\mbox{.}(2023a)]%
        {li2023survey}
\bibfield{author}{\bibinfo{person}{Xiao Li}, \bibinfo{person}{Li Sun}, \bibinfo{person}{Mengjie Ling}, {and} \bibinfo{person}{Yan Peng}.} \bibinfo{year}{2023}\natexlab{a}.
\newblock \showarticletitle{A survey of graph neural network based recommendation in social networks}.
\newblock \bibinfo{journal}{\emph{Neurocomputing}}  \bibinfo{volume}{549} (\bibinfo{year}{2023}), \bibinfo{pages}{126441}.
\newblock


\bibitem[Li et~al\mbox{.}(2024)]%
        {li2024generalized}
\bibfield{author}{\bibinfo{person}{Yibo Li}, \bibinfo{person}{Xiao Wang}, \bibinfo{person}{Hongrui Liu}, {and} \bibinfo{person}{Chuan Shi}.} \bibinfo{year}{2024}\natexlab{}.
\newblock \showarticletitle{A generalized neural diffusion framework on graphs}. In \bibinfo{booktitle}{\emph{Proceedings of the AAAI Conference on Artificial Intelligence}}, Vol.~\bibinfo{volume}{38}. \bibinfo{pages}{8707--8715}.
\newblock


\bibitem[Ling et~al\mbox{.}(2023)]%
        {ling2023deep}
\bibfield{author}{\bibinfo{person}{Chen Ling}, \bibinfo{person}{Junji Jiang}, \bibinfo{person}{Junxiang Wang}, \bibinfo{person}{My~T Thai}, \bibinfo{person}{Renhao Xue}, \bibinfo{person}{James Song}, \bibinfo{person}{Meikang Qiu}, {and} \bibinfo{person}{Liang Zhao}.} \bibinfo{year}{2023}\natexlab{}.
\newblock \showarticletitle{Deep graph representation learning and optimization for influence maximization}. In \bibinfo{booktitle}{\emph{International conference on machine learning}}. PMLR, \bibinfo{pages}{21350--21361}.
\newblock


\bibitem[Liu et~al\mbox{.}(2020)]%
        {liu2020towards}
\bibfield{author}{\bibinfo{person}{Meng Liu}, \bibinfo{person}{Hongyang Gao}, {and} \bibinfo{person}{Shuiwang Ji}.} \bibinfo{year}{2020}\natexlab{}.
\newblock \showarticletitle{Towards deeper graph neural networks}. In \bibinfo{booktitle}{\emph{Proceedings of the 26th ACM SIGKDD international conference on knowledge discovery \& data mining}}. \bibinfo{pages}{338--348}.
\newblock


\bibitem[Liu et~al\mbox{.}(2021)]%
        {liu2021graph}
\bibfield{author}{\bibinfo{person}{Xiaorui Liu}, \bibinfo{person}{Jiayuan Ding}, \bibinfo{person}{Wei Jin}, \bibinfo{person}{Han Xu}, \bibinfo{person}{Yao Ma}, \bibinfo{person}{Zitao Liu}, {and} \bibinfo{person}{Jiliang Tang}.} \bibinfo{year}{2021}\natexlab{}.
\newblock \showarticletitle{Graph neural networks with adaptive residual}.
\newblock \bibinfo{journal}{\emph{Advances in Neural Information Processing Systems}}  \bibinfo{volume}{34} (\bibinfo{year}{2021}), \bibinfo{pages}{9720--9733}.
\newblock


\bibitem[Lyu et~al\mbox{.}(2023)]%
        {lyu2023dcgnn}
\bibfield{author}{\bibinfo{person}{Nuoyan Lyu}, \bibinfo{person}{Bingbing Xu}, \bibinfo{person}{Fangda Guo}, {and} \bibinfo{person}{Huawei Shen}.} \bibinfo{year}{2023}\natexlab{}.
\newblock \showarticletitle{DCGNN: Dual-channel graph neural network for social bot detection}. In \bibinfo{booktitle}{\emph{Proceedings of the 32nd ACM International Conference on Information and Knowledge Management}}. \bibinfo{pages}{4155--4159}.
\newblock


\bibitem[Mannan et~al\mbox{.}(2021)]%
        {mannan2021study}
\bibfield{author}{\bibinfo{person}{Md~Abdul Mannan}, \bibinfo{person}{Md~R Rahman}, \bibinfo{person}{Halima Akter}, \bibinfo{person}{Nazmun Nahar}, {and} \bibinfo{person}{Samiran Mondal}.} \bibinfo{year}{2021}\natexlab{}.
\newblock \showarticletitle{A study of Banach fixed point theorem and it’s applications}.
\newblock \bibinfo{journal}{\emph{American Journal of Computational Mathematics}} \bibinfo{volume}{11}, \bibinfo{number}{2} (\bibinfo{year}{2021}), \bibinfo{pages}{157--174}.
\newblock


\bibitem[Panagopoulos et~al\mbox{.}(2023)]%
        {panagopoulos2023maximizing}
\bibfield{author}{\bibinfo{person}{George Panagopoulos}, \bibinfo{person}{Nikolaos Tziortziotis}, \bibinfo{person}{Michalis Vazirgiannis}, {and} \bibinfo{person}{Fragkiskos Malliaros}.} \bibinfo{year}{2023}\natexlab{}.
\newblock \showarticletitle{Maximizing influence with graph neural networks}. In \bibinfo{booktitle}{\emph{Proceedings of the International Conference on Advances in Social Networks Analysis and Mining}}. \bibinfo{pages}{237--244}.
\newblock


\bibitem[Pei et~al\mbox{.}(2020)]%
        {pei2020geom}
\bibfield{author}{\bibinfo{person}{Hongbin Pei}, \bibinfo{person}{Bingzhe Wei}, \bibinfo{person}{Kevin Chen~Chuan Chang}, \bibinfo{person}{Yu Lei}, {and} \bibinfo{person}{Bo Yang}.} \bibinfo{year}{2020}\natexlab{}.
\newblock \showarticletitle{GEOM-GCN: GEOMETRIC GRAPH CONVOLUTIONAL NETWORKS}. In \bibinfo{booktitle}{\emph{8th International Conference on Learning Representations, ICLR}}.
\newblock


\bibitem[Platonov et~al\mbox{.}(2023)]%
        {platonovcritical}
\bibfield{author}{\bibinfo{person}{Oleg Platonov}, \bibinfo{person}{Denis Kuznedelev}, \bibinfo{person}{Michael Diskin}, \bibinfo{person}{Artem Babenko}, {and} \bibinfo{person}{Liudmila Prokhorenkova}.} \bibinfo{year}{2023}\natexlab{}.
\newblock \showarticletitle{A critical look at the evaluation of GNNs under heterophily: Are we really making progress?}. In \bibinfo{booktitle}{\emph{The Eleventh International Conference on Learning Representations}}.
\newblock


\bibitem[Rainer and Krause(2002)]%
        {rainer2002opinion}
\bibfield{author}{\bibinfo{person}{Hegselmann Rainer} {and} \bibinfo{person}{Ulrich Krause}.} \bibinfo{year}{2002}\natexlab{}.
\newblock \showarticletitle{Opinion dynamics and bounded confidence: models, analysis and simulation}.
\newblock  (\bibinfo{year}{2002}).
\newblock


\bibitem[Robbins and Monro(1951)]%
        {robbins1951stochastic}
\bibfield{author}{\bibinfo{person}{Herbert Robbins} {and} \bibinfo{person}{Sutton Monro}.} \bibinfo{year}{1951}\natexlab{}.
\newblock \showarticletitle{A stochastic approximation method}.
\newblock \bibinfo{journal}{\emph{The annals of mathematical statistics}} (\bibinfo{year}{1951}), \bibinfo{pages}{400--407}.
\newblock


\bibitem[Rossi et~al\mbox{.}(2024)]%
        {rossi2024edge}
\bibfield{author}{\bibinfo{person}{Emanuele Rossi}, \bibinfo{person}{Bertrand Charpentier}, \bibinfo{person}{Francesco Di~Giovanni}, \bibinfo{person}{Fabrizio Frasca}, \bibinfo{person}{Stephan G{\"u}nnemann}, {and} \bibinfo{person}{Michael~M Bronstein}.} \bibinfo{year}{2024}\natexlab{}.
\newblock \showarticletitle{Edge directionality improves learning on heterophilic graphs}. In \bibinfo{booktitle}{\emph{Learning on Graphs Conference}}. PMLR, \bibinfo{pages}{25--1}.
\newblock


\bibitem[Rossi and Ahmed(2015)]%
        {rossi2015network}
\bibfield{author}{\bibinfo{person}{Ryan Rossi} {and} \bibinfo{person}{Nesreen Ahmed}.} \bibinfo{year}{2015}\natexlab{}.
\newblock \showarticletitle{The network data repository with interactive graph analytics and visualization}. In \bibinfo{booktitle}{\emph{Proceedings of the AAAI conference on artificial intelligence}}, Vol.~\bibinfo{volume}{29}.
\newblock


\bibitem[Rusch et~al\mbox{.}(2022)]%
        {rusch2022graph}
\bibfield{author}{\bibinfo{person}{T~Konstantin Rusch}, \bibinfo{person}{Ben Chamberlain}, \bibinfo{person}{James Rowbottom}, \bibinfo{person}{Siddhartha Mishra}, {and} \bibinfo{person}{Michael Bronstein}.} \bibinfo{year}{2022}\natexlab{}.
\newblock \showarticletitle{Graph-coupled oscillator networks}. In \bibinfo{booktitle}{\emph{International Conference on Machine Learning}}. PMLR, \bibinfo{pages}{18888--18909}.
\newblock


\bibitem[Thorpe et~al\mbox{.}(2022)]%
        {thorpegrand++}
\bibfield{author}{\bibinfo{person}{Matthew Thorpe}, \bibinfo{person}{Tan~Minh Nguyen}, \bibinfo{person}{Hedi Xia}, \bibinfo{person}{Thomas Strohmer}, \bibinfo{person}{Andrea Bertozzi}, \bibinfo{person}{Stanley Osher}, {and} \bibinfo{person}{Bao Wang}.} \bibinfo{year}{2022}\natexlab{}.
\newblock \showarticletitle{GRAND++: Graph Neural Diffusion with A Source Term}. In \bibinfo{booktitle}{\emph{International Conference on Learning Representations}}.
\newblock


\bibitem[Vargas-P{\'e}rez et~al\mbox{.}(2024)]%
        {vargas2024unveiling}
\bibfield{author}{\bibinfo{person}{V{\'\i}ctor~A Vargas-P{\'e}rez}, \bibinfo{person}{Jes{\'u}s Gir{\'a}ldez-Cru}, \bibinfo{person}{Pablo Mesejo}, {and} \bibinfo{person}{Oscar Cord{\'o}n}.} \bibinfo{year}{2024}\natexlab{}.
\newblock \showarticletitle{Unveiling Agents’ Confidence in Opinion Dynamics Models via Graph Neural Networks}.
\newblock \bibinfo{journal}{\emph{IEEE Transactions on Computational Social Systems}} (\bibinfo{year}{2024}).
\newblock


\bibitem[Veli{\v{c}}kovi{\'c} et~al\mbox{.}(2018)]%
        {velivckovic2018graph}
\bibfield{author}{\bibinfo{person}{Petar Veli{\v{c}}kovi{\'c}}, \bibinfo{person}{Guillem Cucurull}, \bibinfo{person}{Arantxa Casanova}, \bibinfo{person}{Adriana Romero}, \bibinfo{person}{Pietro Li{\`o}}, {and} \bibinfo{person}{Yoshua Bengio}.} \bibinfo{year}{2018}\natexlab{}.
\newblock \showarticletitle{Graph Attention Networks}. In \bibinfo{booktitle}{\emph{International Conference on Learning Representations}}.
\newblock


\bibitem[Wang et~al\mbox{.}(2025)]%
        {wang2025resolving}
\bibfield{author}{\bibinfo{person}{Keqin Wang}, \bibinfo{person}{Yulong Yang}, \bibinfo{person}{Ishan Saha}, {and} \bibinfo{person}{Christine Allen-Blanchette}.} \bibinfo{year}{2025}\natexlab{}.
\newblock \showarticletitle{Resolving Oversmoothing with Opinion Dissensus}.
\newblock \bibinfo{journal}{\emph{arXiv preprint arXiv:2501.19089}} (\bibinfo{year}{2025}).
\newblock


\bibitem[Wang et~al\mbox{.}(2022)]%
        {wang2022acmp}
\bibfield{author}{\bibinfo{person}{Yuelin Wang}, \bibinfo{person}{Kai Yi}, \bibinfo{person}{Xinliang Liu}, \bibinfo{person}{Yu~Guang Wang}, {and} \bibinfo{person}{Shi Jin}.} \bibinfo{year}{2022}\natexlab{}.
\newblock \showarticletitle{ACMP: Allen-cahn message passing with attractive and repulsive forces for graph neural networks}. In \bibinfo{booktitle}{\emph{The Eleventh International Conference on Learning Representations}}.
\newblock


\bibitem[Wieder et~al\mbox{.}(2020)]%
        {wieder2020compact}
\bibfield{author}{\bibinfo{person}{Oliver Wieder}, \bibinfo{person}{Stefan Kohlbacher}, \bibinfo{person}{M{\'e}laine Kuenemann}, \bibinfo{person}{Arthur Garon}, \bibinfo{person}{Pierre Ducrot}, \bibinfo{person}{Thomas Seidel}, {and} \bibinfo{person}{Thierry Langer}.} \bibinfo{year}{2020}\natexlab{}.
\newblock \showarticletitle{A compact review of molecular property prediction with graph neural networks}.
\newblock \bibinfo{journal}{\emph{Drug Discovery Today: Technologies}}  \bibinfo{volume}{37} (\bibinfo{year}{2020}), \bibinfo{pages}{1--12}.
\newblock


\bibitem[Wu et~al\mbox{.}(2022)]%
        {wu2022graph}
\bibfield{author}{\bibinfo{person}{Shiwen Wu}, \bibinfo{person}{Fei Sun}, \bibinfo{person}{Wentao Zhang}, \bibinfo{person}{Xu Xie}, {and} \bibinfo{person}{Bin Cui}.} \bibinfo{year}{2022}\natexlab{}.
\newblock \showarticletitle{Graph neural networks in recommender systems: a survey}.
\newblock \bibinfo{journal}{\emph{Comput. Surveys}} \bibinfo{volume}{55}, \bibinfo{number}{5} (\bibinfo{year}{2022}), \bibinfo{pages}{1--37}.
\newblock


\bibitem[Wu et~al\mbox{.}(2020)]%
        {wu2020comprehensive}
\bibfield{author}{\bibinfo{person}{Zonghan Wu}, \bibinfo{person}{Shirui Pan}, \bibinfo{person}{Fengwen Chen}, \bibinfo{person}{Guodong Long}, \bibinfo{person}{Chengqi Zhang}, {and} \bibinfo{person}{S~Yu Philip}.} \bibinfo{year}{2020}\natexlab{}.
\newblock \showarticletitle{A comprehensive survey on graph neural networks}.
\newblock \bibinfo{journal}{\emph{IEEE transactions on neural networks and learning systems}} \bibinfo{volume}{32}, \bibinfo{number}{1} (\bibinfo{year}{2020}), \bibinfo{pages}{4--24}.
\newblock


\bibitem[Wu et~al\mbox{.}(2023)]%
        {wu2023chemistry}
\bibfield{author}{\bibinfo{person}{Zhenxing Wu}, \bibinfo{person}{Jike Wang}, \bibinfo{person}{Hongyan Du}, \bibinfo{person}{Dejun Jiang}, \bibinfo{person}{Yu Kang}, \bibinfo{person}{Dan Li}, \bibinfo{person}{Peichen Pan}, \bibinfo{person}{Yafeng Deng}, \bibinfo{person}{Dongsheng Cao}, \bibinfo{person}{Chang-Yu Hsieh}, {et~al\mbox{.}}} \bibinfo{year}{2023}\natexlab{}.
\newblock \showarticletitle{Chemistry-intuitive explanation of graph neural networks for molecular property prediction with substructure masking}.
\newblock \bibinfo{journal}{\emph{Nature communications}} \bibinfo{volume}{14}, \bibinfo{number}{1} (\bibinfo{year}{2023}), \bibinfo{pages}{2585}.
\newblock


\bibitem[Xia et~al\mbox{.}(2021)]%
        {xia2021deepis}
\bibfield{author}{\bibinfo{person}{Wenwen Xia}, \bibinfo{person}{Yuchen Li}, \bibinfo{person}{Jun Wu}, {and} \bibinfo{person}{Shenghong Li}.} \bibinfo{year}{2021}\natexlab{}.
\newblock \showarticletitle{Deepis: Susceptibility estimation on social networks}. In \bibinfo{booktitle}{\emph{Proceedings of the 14th ACM International Conference on Web Search and Data Mining}}. \bibinfo{pages}{761--769}.
\newblock


\bibitem[Xu et~al\mbox{.}(2018)]%
        {xu2018representation}
\bibfield{author}{\bibinfo{person}{Keyulu Xu}, \bibinfo{person}{Chengtao Li}, \bibinfo{person}{Yonglong Tian}, \bibinfo{person}{Tomohiro Sonobe}, \bibinfo{person}{Ken-ichi Kawarabayashi}, {and} \bibinfo{person}{Stefanie Jegelka}.} \bibinfo{year}{2018}\natexlab{}.
\newblock \showarticletitle{Representation learning on graphs with jumping knowledge networks}. In \bibinfo{booktitle}{\emph{International conference on machine learning}}. PMLR, \bibinfo{pages}{5453--5462}.
\newblock


\bibitem[Xu et~al\mbox{.}(2022)]%
        {xu2022effects}
\bibfield{author}{\bibinfo{person}{Wanyue Xu}, \bibinfo{person}{Liwang Zhu}, \bibinfo{person}{Jiale Guan}, \bibinfo{person}{Zuobai Zhang}, {and} \bibinfo{person}{Zhongzhi Zhang}.} \bibinfo{year}{2022}\natexlab{}.
\newblock \showarticletitle{Effects of stubbornness on opinion dynamics}. In \bibinfo{booktitle}{\emph{Proceedings of the 31st ACM International Conference on Information and Knowledge Management}}. \bibinfo{pages}{2321--2330}.
\newblock


\bibitem[Zhang et~al\mbox{.}(2023)]%
        {zhang2023drgcn}
\bibfield{author}{\bibinfo{person}{Lei Zhang}, \bibinfo{person}{Xiaodong Yan}, \bibinfo{person}{Jianshan He}, \bibinfo{person}{Ruopeng Li}, {and} \bibinfo{person}{Wei Chu}.} \bibinfo{year}{2023}\natexlab{}.
\newblock \showarticletitle{Drgcn: Dynamic evolving initial residual for deep graph convolutional networks}. In \bibinfo{booktitle}{\emph{Proceedings of the AAAI conference on artificial intelligence}}, Vol.~\bibinfo{volume}{37}. \bibinfo{pages}{11254--11261}.
\newblock


\bibitem[Zhao et~al\mbox{.}(2021)]%
        {zhao2021adaptive}
\bibfield{author}{\bibinfo{person}{Jialin Zhao}, \bibinfo{person}{Yuxiao Dong}, \bibinfo{person}{Ming Ding}, \bibinfo{person}{Evgeny Kharlamov}, {and} \bibinfo{person}{Jie Tang}.} \bibinfo{year}{2021}\natexlab{}.
\newblock \showarticletitle{Adaptive diffusion in graph neural networks}.
\newblock \bibinfo{journal}{\emph{Advances in neural information processing systems}}  \bibinfo{volume}{34} (\bibinfo{year}{2021}), \bibinfo{pages}{23321--23333}.
\newblock


\bibitem[Zhao et~al\mbox{.}(2023)]%
        {zhao2023graph}
\bibfield{author}{\bibinfo{person}{Kai Zhao}, \bibinfo{person}{Qiyu Kang}, \bibinfo{person}{Yang Song}, \bibinfo{person}{Rui She}, \bibinfo{person}{Sijie Wang}, {and} \bibinfo{person}{Wee~Peng Tay}.} \bibinfo{year}{2023}\natexlab{}.
\newblock \showarticletitle{Graph neural convection-diffusion with heterophily}. In \bibinfo{booktitle}{\emph{Proceedings of the Thirty-Second International Joint Conference on Artificial Intelligence}}. \bibinfo{pages}{4656--4664}.
\newblock


\bibitem[Zhou et~al\mbox{.}(2024)]%
        {zhouodnet}
\bibfield{author}{\bibinfo{person}{Bingxin Zhou}, \bibinfo{person}{Outongyi Lv}, \bibinfo{person}{Jing Wang}, \bibinfo{person}{Xiang Xiao}, {and} \bibinfo{person}{Weishu Zhao}.} \bibinfo{year}{2024}\natexlab{}.
\newblock \showarticletitle{ODNet: Opinion Dynamics-Inspired Neural Message Passing for Graphs and Hypergraphs}.
\newblock \bibinfo{journal}{\emph{Transactions on Machine Learning Research}} (\bibinfo{year}{2024}).
\newblock


\end{thebibliography}
